\documentclass{article}

    \PassOptionsToPackage{numbers, compress}{natbib}

 \usepackage [preprint]{neurips_2025}

\usepackage[utf8]{inputenc} 
\usepackage[T1]{fontenc}    
\usepackage{hyperref}       

\usepackage{url}            
\usepackage{booktabs}       
\usepackage{newunicodechar}
\newunicodechar{，}{,}
\usepackage{amsfonts}       
\usepackage{nicefrac}       
\usepackage{microtype}      
\usepackage{xcolor}         
\usepackage{graphicx} 
\usepackage{amsmath}
\usepackage{amssymb}
\usepackage{mathtools}
\usepackage{amsthm}
\usepackage{cleveref}       
\usepackage{mathrsfs}
\usepackage{algorithm}
\usepackage{algorithmic}
\usepackage{geometry}
\usepackage{array}
\usepackage{makecell}
\usepackage{caption}
\usepackage{adjustbox}
\usepackage{enumitem}

\theoremstyle{plain}
\newtheorem{theorem}{Theorem}[section]
\newtheorem{proposition}[theorem]{Proposition}
\newtheorem{lemma}[theorem]{Lemma}

\theoremstyle{definition}
\newtheorem{definition}[theorem]{Definition}
\newtheorem{assumption}[theorem]{Assumption}
\theoremstyle{remark}
\newtheorem{remark}[theorem]{Remark}

\title{Adaptive Sentencing Prediction with Guaranteed Accuracy and Legal Interpretability}

\author{%
    Yifei Jin\thanks{These authors contributed equally to this work.} \\
  SAIS, AMSS \\
  Chinese Academy of Sciences \\
  Beijing, China \\
  \texttt{jinyifei@amss.ac.cn} \\
  \And
  Xin Zheng\footnotemark[1] \\
  AMSS \\
  Chinese Academy of Sciences \\
  Beijing, China \\
  \texttt{zhengxin@amss.ac.cn} \\
  \And
  Lei Guo \\
  AMSS \\
  Chinese Academy of Sciences \\
  Beijing, China \\
  \texttt{lguo@iss.ac.cn} \\
}

\begin{document}

\maketitle

\begin{abstract}

Existing research on judicial sentencing prediction predominantly relies on end-to-end models, which often neglect the inherent sentencing logic and lack interpretability-a critical requirement for both scholarly research and judicial practice. To address this challenge, we make three key contributions:First, we propose a novel Saturated Mechanistic Sentencing (SMS) model, which provides inherent legal interpretability by virtue of its foundation in China's Criminal Law. We also introduce the corresponding Momentum Least Mean Squares (MLMS) adaptive algorithm for this model. Second, for the MLMS algorithm based adaptive sentencing predictor, we establish a mathematical theory on the accuracy of adaptive prediction without resorting to any stationarity and independence assumptions on the data. We also provide a best possible upper bound for the prediction accuracy achievable by the best predictor designed in the known parameters case. Third, we construct a Chinese Intentional Bodily Harm (CIBH) dataset. Utilizing this real-world data, extensive experiments demonstrate that our approach achieves a prediction accuracy that is not far from the best possible theoretical upper bound, validating both the model's suitability and the algorithm's accuracy.
\end{abstract}

\section{Introduction}
\label{submission}

\begin{figure}[htbp]
    \centering
    \includegraphics[width=0.9\linewidth]{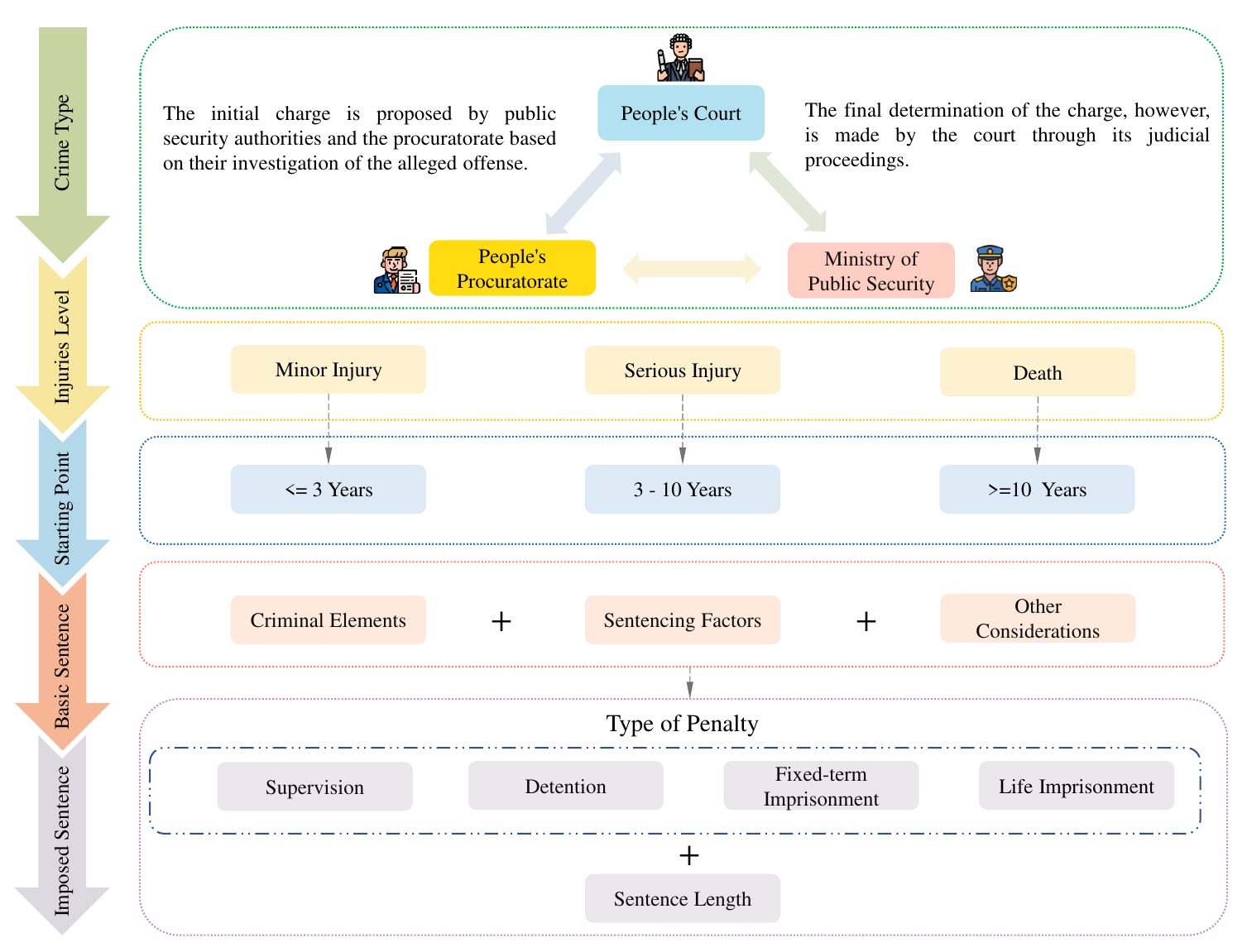}
    \caption{Sentencing Process for Intentional Bodily Harm Cases}
    \label{fig:enter1121}
\end{figure}    

In recent years, with the rapid development of artificial intelligence, Legal Judgment Prediction (LJP) has become a key research area at the intersection of AI and law \cite{aletras2016predicting}. These advancements hold the potential not only to enhance judicial efficiency but also to promote fairness and transparency within the legal system. Sentencing prediction has long been one of the most challenging tasks in LJP, due to the complexity of capturing underlying mechanisms \cite{lyu2023multi, yang2022mve, wu2020integrating}, the lack of theoretically grounded algorithms \cite{yao2020gated}, and the limited availability of relevant data \cite{feng2022legal, medvedeva2023legal}. In this work, we use China's judicial sentencing prediction as a case study, proposing a solution to address these challenges in sentencing prediction, with potential implications for broader applications. 

Firstly, sentencing follows a well-structured “three-main-stage” judicial framework in China: establishing a sentencing starting point, determining a baseline sentence based on circumstances, and finalizing the term through adjustments informed by legal features. This rule-based, interpretable logic provides a strong foundation for developing structured and explainable sentencing prediction models, as we aim to explore in this work. However, most of existing research on judicial sentencing prediction often overlooked the structured judicial logic inherent in China's legal system. For example, although Chen et al. \cite{chen2019charge} proposed a charge-based sentencing prediction method, demonstrating that analyzing individual charges before integrating results could enhance accuracy, this approach only partially captured charge-sentence correlations without fully addressing systemic judicial reasoning. Similarly, Yue et al. \cite{yue2021neurjudge} introduced NeurJudge, a neural framework that decomposes case facts into specific elements to improve prediction interpretability. Despite its advancements in fact segmentation, NeurJudge lacked robust integration with China's judicial logic. Recent studies have made strides toward bridging this gap. For example, Huang \cite{Huang} enhanced interpretability in Chinese legal judgment prediction through legal language modeling, while Sun et al. \cite{SUN2024122177} developed KnowPrompt4LJP, leveraging prompt templates to align pre-trained language models with legal tasks. Nevertheless, these efforts \cite{zhang2024hd} still fall short of systematically incorporating China's structured judicial reasoning. This gap undermines model interpretability and, particularly in data-scarce scenarios, limits the ability to capture the nuanced complexity of judicial decision-making, and may ultimately impact prediction performance \cite{staliunaite2024comparative}.  Fortunately, Wang et al. \cite{wang2022applications} has established a saturated model (S-model) that aligns with judicial logic. However, the prioritization of conviction features as prescribed by law is not reflected in the S-model.

Secondly, the continuous and expanding accumulation of legal judgment documents poses significant challenges for sentencing prediction systems. Therefore, offline batch learning methods are not well-suited to such dynamic environments due to their reliance on full dataset access and reprocessing, which becomes computationally prohibitive as data volumes grow \citep{gama2014survey}. Moreover, alternative buffer-based or multi-pass online methods face limitations in non-stationary legal domains where future data volume is unknown and distribution shifts occur due to evolving societal norms and legal interpretations \cite{gao2013one,zhao2011online}. These distributional changes render static models ineffective over time \cite{kolter2007dynamic, minku2009impact}. To address these challenges, online learning approaches that enable models to adapt incrementally to new data without complete retraining are essential \citep{cesa2021online}. 
Fortunately, the Least Mean Squares (LMS) algorithm \cite{guo1997necessary}, are particularly appealing due to their simplicity, resilience to noise, and independence from prior statistical assumptions. Existing analysis indicates that Momentum LMS offers a faster convergence rate and a smoother convergence trajectory compared to the traditional LMS algorithm \cite{roy1990analysis}. Building upon these insights, we will introduce a novel Momentum LMS algorithm specifically designed for legal prediction tasks.

Thirdly, establishing reliable theoretical bounds on prediction error for online and non-i.i.d. (independent and identically distributed) data—such as the sentencing data—is important, yet it remains a challenging problem. Most existing machine learning theories on prediction error focus on offline data; specifically, they provide upper bounds on the prediction error of a fixed-parameter model evaluated on a given test set. These theories may have limited applicability in online data scenarios. Moreover, such theoretical results often rely on the idealized i.i.d. assumption, which is frequently violated in real-world scenarios \cite{shalev2014understanding, mohri2018foundations}. Therefore, establishing theoretical bounds on prediction error for online and non-i.i.d. sentencing data warrants further investigation.

To address the aforementioned challenges, we use intentional bodily harm as a case study, developing a mechanism model aligned with sentencing logic and designing a corresponding Momentum LMS algorithm. We demonstrate the reliability of both the model and algorithm on a self-constructed dataset of intentional bodily harm cases. In summary, the main contributions of this paper are as follows:

\begin{itemize}

\item 
Based on the "three-step" sentencing logic, we introduce a refined logic-based Saturated Mechanism Sentencing (SMS) Model. Unlike the existing S-model \cite{wang2022applications}, our approach incorporates the legal priority order of conviction-related sentencing features, ensuring better alignment with real-world judicial practices. We design a Momentum Least Mean Squares (MLMS) algorithm , which operates in an online learning framework. This design enables efficient handling of parameter drift—a common challenge in judicial sentencing—while maintaining high computational advantages.
\item 
For the proposed sentencing model and adaptive algorithm, we establish a theoretical upper bound on adaptive prediction accuracy and derive the best possible upper bound in the ideal case of known parameters, both under the general data conditions. Our theoretical framework is broadly applicable, extending the impact of our theoretical findings beyond judicial sentencing. 
\item 
We construct an up-to-date Chinese Intentional Bodily Injury (CIBH) dataset and conduct extensive experiments to evaluate our approach. The results demonstrate that our model and algorithm consistently outperform standard baselines, with adaptive prediction accuracy close to the best theoretical upper bound, validating both the effectiveness and robustness of our method.
To the best of our knowledge, this appears to be the first such kind of results on judicial sentencing in the literature.

\end{itemize}

The remainder of this paper is organized as follows. In section \ref{classification and prediction}, we will present the proposed logic-based sentencing prediction model along with the corresponding optimization algorithm.  Section \ref{Experiment} will introduce the newly constructed real-world dataset and report empirical results, followed by some concluding remarks in Section \ref{conclusion}.

\section{Sentence Length Prediction} \label{classification and prediction}

In this section, we firstly construct a logic‑based sentencing prediction model for real‑world judicial sentencing, then we introduce the corresponding momentum Least Mean Squares (MLMS) algorithm. Finally, we present the main theorems.

\subsection{ Saturated Mechanistic Sentencing Model}
To align with the sentencing process in judicial practice, we introduce a new saturated mechanistic sentencing (SMS) model. Although the SMS model was previously mentioned in a domestic workshop by us (\cite{guo2024integration}), this paper marks its first formal publication. To be specific, the SMS model is described as follows:
\begin{equation}\label{smodel}
\begin{aligned}
y_{k+1}=&S_k \big( \underbrace{ (\overbrace{a_k}^{\parbox{1cm}{ \centering \scriptsize starting point}} +\overbrace{b x_k^{(1)}+c x_k^{(2)}+d x_k^{(3)} + e x_k^{(4)}}^{\text{increase in sentence}})}_{\text{baseline sentence}} \\
&\times\underbrace{\prod_{i=1}^{m_1}(1 + p_i z_k^{(i)})}_{\parbox{2cm}{\centering \scriptsize impact of conviction-related sentencing features}}\times
\underbrace{(1+ \sum_{j=1}^{m_2}q_jv_k^{(j)} + \eta)}_{\text{ impact of other sentencing features }}+\varepsilon_{k+1}\big),
\end{aligned}
\end{equation}
\noindent where the saturation function \(S_k(\cdot)\)  is defined as follows:
\begin{gather}
S_k(x)=
\begin{cases}
U_k, & x>U_k, \\
x, & L_k \leq x \leq U_k, ~~ k =~0,~1, ~2, \cdots,\\
L_k, & x<L_k,
\end{cases}
\end{gather}
\noindent and where \(y_{k+1} \in \mathbb{R}\) denotes the sentence length for the case labeled as k, ordered chronologically. \(a_k \in \mathbb{R}\) is the starting point of sentencing; \(x_k^{(i)}\in \mathbb{R}\) \((i=1,2,3,4)\) represent the number of victims with slight injuries, minor injuries, serious injuries, fatalities, respectively; \(b,c,d,e\in \mathbb{R}\) are real numbers that quantify the weights of the corresponding additional sentence added to the offender for each additional victim with different levels of injury; 
\(m_1\) is the number of conviction-related sentencing features with legal priority order, and \( z_k^{(i)} \) represent the corresponding conviction features specified in the sentencing guidelines, \(p_i\) are the corresponding unknown weighting parameters; \(m_2\) is the number of other sentencing features specified in law, and \( v_k^{(j)} \) represent other sentencing features, and \(q_j\) are the corresponding unknown weighting parameters; \(\eta \in \mathbb{R}\) is a bias term, representing the comprehensive influence of other possible sentencing factors not included in \(x_k^{(1)}\), \(x_k^{(2)}\), \(x_k^{(3)}\), \(x_k^{(4)}\), \( z_k^{(i)} \) and \( v_k^{(j)} \); \(\varepsilon_k \in \mathbb{R}\) represent a possible random noise effect; \(U_k \in \mathbb{R}\) and \(L_k \in \mathbb{R}\) are the upper bound and the lower bound of the announced sentence, respectively, which are prescribed in the Criminal Law of the People's Republic of China and will be detailed in \cref{Experiment}.
\begin{remark}

Our model achieves inherent judicial interpretability through its structural design, as highlighted in \eqref{smodel}, which improves upon the S-model \cite{wang2022applications}. Specifically, this interpretability stems from several key features:

\begin{itemize}[left=0pt]
\item Strict adherence to the three-main-stage-sentencing logic in the Criminal Law of China and the related guidelines: legal starting point → baseline sentence → prioritized adjustments based on conviction features .
\item Employment of saturated functions to ensure predictions align with statutory sentencing ranges given in the law.
\item Compliance of sentencing factor application (sequence and effectiveness) with the law.
\end{itemize}

This structural embedding of legal principles, particularly the multiplicative interaction of conviction-related features on the sentencing benchmark (a nuance missed by the S-model), and the real-world interpretability of parameter signs and magnitudes regarding their impact on sentence length, distinguishes the SMS model. Consequently, unlike black-box models like neural networks or standard regression models (e.g., logistic regression relying on mere linear combinations), the SMS model possesses inherent legal and procedural interpretability. By conforming to judicial logic, our model yield reliable and understandable results for judges, addressing the judiciary's emphasis on fairness and reliability—a guarantee absent in statistical model and black-box models that lack inherent interpretability.

\end{remark}

\subsection{Algorithm Design}\label{MLMS}

Developing an effective and theoretically grounded adaptive algorithm for the SMS model is crucial for achieving enhanced predictive performance. However, the significant nonlinearity inherent in the current mechanistic model \eqref{smodel} presents substantial obstacles in designing and rigorously analyzing such algorithms, particularly when dealing with general data whose statistical properties deviate far from the idealized i.i.d. setting. To overcome this theoretical challenge, we first perform an equivalent transformation of the model within the saturated function into a ``linear regression" form. Drawing specifically on the methodology proposed in \cite{wang2022applications}, we construct a new high-dimensional regression vector $\varphi_k$ by fully expanding the product terms within the saturation function $S_k(\cdot)$. We then introduce the corresponding unknown weighting parameter vector $\theta$, such that the inner product $\varphi_k^{\top}\theta$ precisely matches the original expression inside the saturation function. The transformed model, resulting from this reformulation, is expressed as follows:
\begin{equation} \label{slinear1}
y_{k+1}=S_k\left(\varphi_k^{\top} \theta+ \varepsilon_{k+1}\right), \quad k=0,1,2, \cdots,
\end{equation}
where $\varepsilon_{k+1} \in \mathbb{R}$ is the noise, $\varphi_k \in \mathbb{R}^p~(p \geq 1)$ and  $\theta \in \mathbb{R}^p$ are specified as the follows:

\begin{equation}
\begin{aligned}
\varphi_k= & {\left[a_k \phi_{1 k}^{\top}, ~a_k\left(\phi_{2 k} \otimes \phi_{1 k}\right)^{\top}, ~x_k^{(1)} \phi_{1 k}^{\top},~  x_k^{(1)}\left(\phi_{2 k} \otimes \phi_{1 k}\right)^{\top}, ~x_k^{(2)} \phi_{1 k}^{\top}, ~x_k^{(2)}\left(\phi_{2 k} \otimes \phi_{1 k}\right)^{\top}, \right.} \\
& \left. x_k^{(3)} \phi_{1 k}^{\top}, ~x_k^{(3)}\left(\phi_{2 k} \otimes \phi_{1 k}\right)^{\top} , x_k^{(4)} \phi_{1 k}^{\top},~x_k^{(4)} \left(\phi_{2 k} \otimes \phi_{1 k}\right)^{\top}\right]^{\top}, \\
\theta =& \big[(1+\eta)\vartheta_1^{\top}, (\vartheta_2 \otimes \vartheta_1)^{\top}, b(1+\eta)\vartheta_1^{\top}, 
	 b(\vartheta_2 \otimes \vartheta_1)^{\top},  c(1+\eta)\vartheta_1^{\top}, c(\vartheta_2 \otimes \vartheta_1)^{\top},\\
         &d(1+\eta)\vartheta_1^{\top}, d(\vartheta_2 \otimes \vartheta_1)^{\top}, e(1+\eta)\vartheta_1^{\top}, e(\vartheta_2 \otimes \vartheta_1)^{\top}
	\big]^{\top},
\end{aligned}
\end{equation}
where
\begin{equation}
\begin{aligned}
\phi_{1k} &= \left[
1,\,
z_k^{(1)},\, z_k^{(2)},\, \dots,\, z_k^{(m_1)},\,
z_k^{(1)}z_k^{(2)},\, z_k^{(1)}z_k^{(3)},\, \dots,\, z_k^{(m_1-1)}z_k^{(m_1)},\,
\dots,\,
z_k^{(1)}z_k^{(2)}\cdots z_k^{(m_1)}
\right]^{\top},\\
\phi_{2k} &= \left[
v_k^{(1)},\, v_k^{(2)},\, \dots,\, v_k^{(m_2)}
\right]^{\top}, \\
 \vartheta_1 &=[
	1, \ p_1, \ \cdots, \ p_{m_1}, \ p_1 p_2, \ \cdots, \ p_{m_1-1} p_{m_1}, \cdots, \ p_1 \cdots p_{m_1} ]^{\top}, \\
\vartheta_2 &= [ q_1, \ \cdots, \ q_{m_2} ]^{\top},
\end{aligned}
\end{equation}
and $\otimes$ denotes the Kronecker product between two vectors.

In practice, as the legal system continues to evolve, the sentencing guidelines may undergo minor revisions to adjust the influence of certain sentencing features on the length of imprisonment. Such adjustments may cause (parts of) the parameter vector $\theta$, which captures the effect of these features, to vary over time. Accordingly, we consider a time-varying parameter vector $\theta_k$, then \eqref{slinear1} becomes the following saturated time-varying parameter model:
\begin{equation} \label{saturation time varying}
y_{k+1}=S_k\left(\varphi_k^{\top} \theta_k+ \varepsilon_{k+1}\right), \quad k=0,1,2,\cdots
\end{equation}

In order to attain theoretically guaranteed sentence length prediction performance for model \eqref{saturation time varying}, we need to design the corresponding algorithm and establish a related theory, for which we introduce the following notations and assumptions:

\textbf{Notations and Assumptions}

\noindent\textbf{Notations.}  $\|\cdot\|$ denotes the Euclidean norm for vectors and matrices. $|\cdot|$ is the $\ell_1$-norm for vectors, i.e., the sum of the absolute values of its components. For a matrix $M$, $M^{\top}$ denotes its transpose. Let $\{\mathcal{F}_k, k \geq 0\}$ be a non-decreasing sequence of $\sigma$-algebras, and let $\mathbb{E}[\cdot \mid \mathcal{F}_k]$ denote the conditional mathematical expectation. For simplicity, we may use $\mathbb{E}_k[\cdot]$ or $\mathbb{E}_k\{\cdot\}$ to place $\mathbb{E}[\cdot \mid \mathcal{F}_k]$. A sequence $\{x_k, \mathcal{F}_k, k \geq 0\}$ is adapted if $x_k$ is $\mathcal{F}_k$-measurable for all $k \geq 0$.

\begin{assumption}\label{A1}
The adapted sequence \( \{ \varphi_k, \mathcal{F}_k, k \geq 0\}\) is bounded, i.e., \(\sup\limits_{k \geq 0} \| \varphi_k\| < \infty\), and the parameter vector \(\theta_k\) belongs to a known convex compact set $\mathcal{D}=\left\{x=\left(x_1, \cdots, x_{p}\right)^{\mathrm{T}} \in \mathscr{R}^{p},\left|x_i\right| \leqslant L, 1 \leqslant i \leqslant p\right\}$. Moreover, the parameter sequence ${\theta_k}$ is slowly time-varying in the following time-averaged sense, i.e.,
\begin{equation} \label{bountheta}
\limsup_{n-m \to \infty} \frac{1}{n-m} \sum_{i=m+1}^n \|\theta_i - \theta_{i-1}\| \leq \xi,
\end{equation}
where $\xi \in (0,1)$ is a positive constant.
\end{assumption}

We remark that  the above condition is general than the following condition:
\begin{equation} \label{bountheta2}
\left\|\theta_i - \theta_{i-1}\right\| \leq \xi, \quad \forall i \geq 1.
\end{equation}
 Moreover, one can  deduce that there is a bounded random sequence $\{C_k, k\geq 0\}$ such that
\begin{equation}
    \sup_{\theta \in \mathcal{D}} |\varphi_k^{\top}\theta| \leq C_k.
\end{equation}

\begin{assumption}\label{A2}
The thresholds $\{L_k, \mathcal{F}_k, k\geq 0\}$ and $\{U_k, \mathcal{F}_k, k\geq 0\}$ are known adapted sequences. Moreover, there exists positive constants $c$ and $M$ such that $0<c \leq L_k < U_k \leq M<\infty$ for all $k\geq 0$. 
\end{assumption}
Assumption \ref{A2} is quite general, requiring $y_k$ to be positive and bounded. The judicial sentencing data in China inherently satisfy this requirement.

\begin{assumption}\label{A3}
The conditional expectation function $G_k(x)=\mathbb{E}_k\left[S_k\left(x+ \varepsilon_{k+1}\right)\right]$ is known and differentiable with its derivative denoted by $G_k^{\prime}(\cdot)$. Furthermore, there exists a constant $C>\sup\limits_{k>0} C_k$ such that
\begin{equation}
\inf _{|x| \leq C, ~k \geq 0} G_k^{\prime}(x) >0,~ \text { a.s. }
\end{equation}
For notational simplicity, we denote
\begin{equation}
\inf _{|x| \leq C_k} G_k^{\prime}(x)=\underline{g}_k, \sup _{|x| \leq C_k} G_k^{\prime}(x)=\bar{g}_k .
\end{equation}
\end{assumption}

\begin{remark}
From the definition of $S_k(\cdot)$, we know that $S_k^{\prime}(\cdot)=0$ or $ 1$ almost surely, so that $ \sup\limits_{k\geq 0}G_k^{\prime}(\cdot)\leq 1$ almost surely  by a direct calculation. Moreover, it is worth noting that $G_k(\varphi_k^{\top}\theta_k)$ is the mean‑squared‑error optimal predictor of $y_{k+1}$ when $\theta_k$ is known. However, since the true parameter $\theta_k$ is not known \textit{a priori} in our work, it needs to be estimated. Upon obtaining an estimate $\hat{\theta}_k$, we replace $\theta_k$ with $\hat{\theta}_k$ in the optimal predictor, yielding the adaptive predictor $G_k(\varphi_k^{\top} \hat{\theta}_k)$.
\end{remark}

\begin{assumption} \label{bouned noise}
The conditional density function $f_{k+1}(\cdot)$ of the noise $\varepsilon_{k+1}$ given the $\sigma$-algebra $\mathcal{F}_k$ satisfies the following lower bound condition, i.e, for any fixed $h > 0$, there exists a constant $\epsilon_{h} > 0$ such that
\begin{equation}
    \inf_{|x| \leq h} f_{k+1}(x) \geq \epsilon_{h}, \quad \text{a.s.}
\end{equation}
\end{assumption}

\textbf{The Momentum Least Mean Squared Algorithm}

To design the algorithm, we introduce the following projection operator.
\begin{definition}\label{projection}
The projection operator $\Pi_{\mathcal{D}}(\cdot)$ is defined as
\begin{equation}
   \pi_{\mathcal{D}}(x)=\underset{y \in \mathcal{D}}{\arg \min }\|x-y\|, \quad \forall x \in \mathbb{R}^p ,
\end{equation}
where $\mathcal{D}$ is the convex compact set introduced in Assumption \ref{A1}.
\end{definition}
We are now ready to introduce the following algorithm.

\begin{algorithm}[H]
\caption{Adaptive Momentem LMS}\label{prediction}
\label{alg:amlms-p}
    \textbf{Input:} Data $\{\varphi_k, y_k\}_{k=1}^{T}$; Step size $\mu \in (0, 1]$; Momentum coefficient $\beta_k = \frac{1}{k^{\delta}}, \delta \in (0,1)$; Projection set $\mathcal{D}$.

\textbf{Output:} Prediction $\{\hat{y}_k\}_{k=1}^{T}$; Parameter estimates $\{\hat{\theta}_k\}_{k=1}^{T}$.

\begin{algorithmic}[1]\label{algorithm1}
    \STATE Initialize: $\hat{\theta}_0 = \mathbf{0}$, $\hat{\theta}_{-1} = \mathbf{0}$.
    \FOR{$k = 0, 1, 2, 3, \cdots$ }
        \STATE Compute adaptive learning rate:
        $$
        \alpha_k = \frac{\mu}{1 + \|\varphi_k\|^2}
        $$
        \STATE Compute gradient:
        $$
        g_k = - \left(y_{k+1}-G_k\left(\varphi_k^{\top} \hat{\theta}_k\right)\right) G_k^{\prime}\left(\varphi_k^{\top} \hat{\theta}_k\right) \varphi_k
        $$
        \STATE Parameter update:
        $$
        \hat{\theta}_{k+1} = \pi_D\left[ \hat{\theta}_k - \alpha_k g_k + \beta_k (\hat{\theta}_k - \hat{\theta}_{k-1}) \right]
        $$
        \STATE Prediction:
        $$
        \hat{y}_{k+1} = S_k(\varphi_{k}^\mathrm{T} \hat{\theta}_{k})
        $$
    \ENDFOR
\end{algorithmic}
\end{algorithm}
\begin{remark}
Algorithm \ref{algorithm1} is an adaptive one because it updates the next parameter estimate $\hat{\theta}_{k+1}$ only based on the current parameter estimate $\hat{\theta}_k$ and the newly acquired online data $\{\varphi_k, y_{k+1}\}$. Notably, the index $k$ herein represents sequential time updates driven by the received data stream, distinguishing it from the iterative search procedures employed in offline optimization algorithms.

\end{remark}

\subsection{Main Theorems}\label{MAIN THEOREM}

We now establish theoretical upper and lower bounds on the prediction accuracy, respectively. The accuracy metric is defined as an averaged relative prediction accuracy:
\begin{equation}\label{RAmetric}
  \operatorname{RA}(T) = 1 - \frac{1}{T} \sum_{k=1}^T \frac{\left|y_{k+1} - \hat{y}_{k+1}\right|}{y_{k+1}},
\end{equation}
where $T$ denotes the number of cases that have arrived online, $y_k$ is the true sentence length for the $k$-th case and is uniformly bounded below by a positive threshold specified in law,  $\hat{y}_k$ is the corresponding adaptive predictor.

We first give a lower bound to the accuracy of the adaptive predictor $\hat{y}_{k+1} = G_k'(\varphi_k^\top \hat{\theta}_k)$.

\begin{theorem}\label{main theorem}
Under Assumptions \ref{A1}--\ref{A3}, we have the following lower bound to the accuracy of the adaptive predictor:

\begin{equation}\label{bound error upper}
     \limsup _{T \rightarrow \infty} RA(T) \geq  1- \frac{\mathcal{C}(\sigma_w +\xi^{\frac{1}{2}})}{c},~ \text{a.s.}
\end{equation}

where $\sigma_w^2 = \limsup\limits_{k\geq 0} \mathbb{E}[w_{k+1}^2 \mid \mathcal{F}_{k}]$, $w_{k+1}=y_{k+1}-G_k\left(\varphi_k^{\top} \theta_k\right)$ and $\mathcal{C}$ is a positive constant which does not dependent on $\sigma_w^2$ and $\xi$. 
\end{theorem}

\begin{remark}
We remark that the inequality \eqref{bound error upper} does not require any stationarity and independence assumptions on the data, and the lower bound of $RA(T)$ only relies on the variance bound $\sigma_w$ of the martingale difference sequence $\{w_k\}$ and parameter variation bound $\xi$. It is evident that the closer the values of $\sigma_w$ and $\xi$ are to 0, the closer the metric $RA(T)$ of the adaptive predictor $\hat{y}_{k}$ is to 1. 
\end{remark}

After establishing the lower bound to the accuracy of the adaptive predictor, a natural question is whether the prediction accuracy can reach 1 when applied to real data? If not,  what is the maximum achievable accuracy for this dataset? Regarding the first question, Proposition \ref{Claim 1} in Appendix \ref{Proof of main theorems} demonstrates that, if the noise variances are uniformly bounded below by a positive constant, then \eqref{RAmetric} cannot attain a value of 1, regardless of the predictors used. As for the second question, we provide a theorem that provides a best possible upper bound for $RA(T)$ achievable by the optimal predictor when the parameter is known.

\begin{theorem}\label{theorem new}
Under the Assumption~\ref{bouned noise} and conditions of Theorem~\ref{main theorem}, a best possible upper bound for $RA(T)$ based on the best predictor $\hat{y}_{k+1}^* = S_k(\varphi_{k}^{\top}\theta_{k})$ in the ideal case of known parameter $\theta_k$, can be computed as:
\begin{equation}
    \limsup_{T \rightarrow \infty} RA^*(T) = \limsup_{T \rightarrow \infty} \left(1 - \frac{1}{T} \sum_{k=1}^T \frac{\sigma_{sk}}{y_{k+1}} \right), \quad \text{a.s.}
\end{equation}
where $RA^*(T) = 1- \frac{1}{T}\sum_{k=1}^T \frac{\left|y_{k+1} - \hat{y}^*_{k+1}\right|}{y_{k+1}}$ is the relative accuracy of the best predictor $\hat{y}_{k+1}^*$, $\sigma_{k} = \mathbb{E}_{k}\left[|S_k(\varphi_{k}^{\top}\theta_{k} + \varepsilon_{k+1})-S_k(\varphi_{k}^{\top}\theta_{k} )|\right]$. 
\end{theorem}
The proof of Theorem \ref{theorem new} is provided in Appendix \ref{Proof of main theorems}, where an illustrative example that demonstrates how to compute $\sigma_{k}$ is also provided. Obviously, based on Theorem \ref{theorem new}, one can compute the theoretically best achievable prediction accuracy of $RA^*(T)$ from the available data.

\section{Real-World Dataset Construction and Performance Evaluation} \label{Experiment}

In this section, we firstly introduce the Chinese Intentional Bodily Harm (CIBH) dataset. Secondly, we present the evaluation metrics commonly employed in legal contexts and theoretically demonstrate the relationship between relative accuracy's upper bound and sample noise. Finally, we evaluate the performance of our proposed Saturated Mechanism Sentencing (SMS) model and the Momentum LMS algorithm on this dataset.

\subsection{Dataset Overview}\label{Dataset Overview}

\begin{figure*}[h]
    \centering
    \includegraphics[width=1\linewidth]{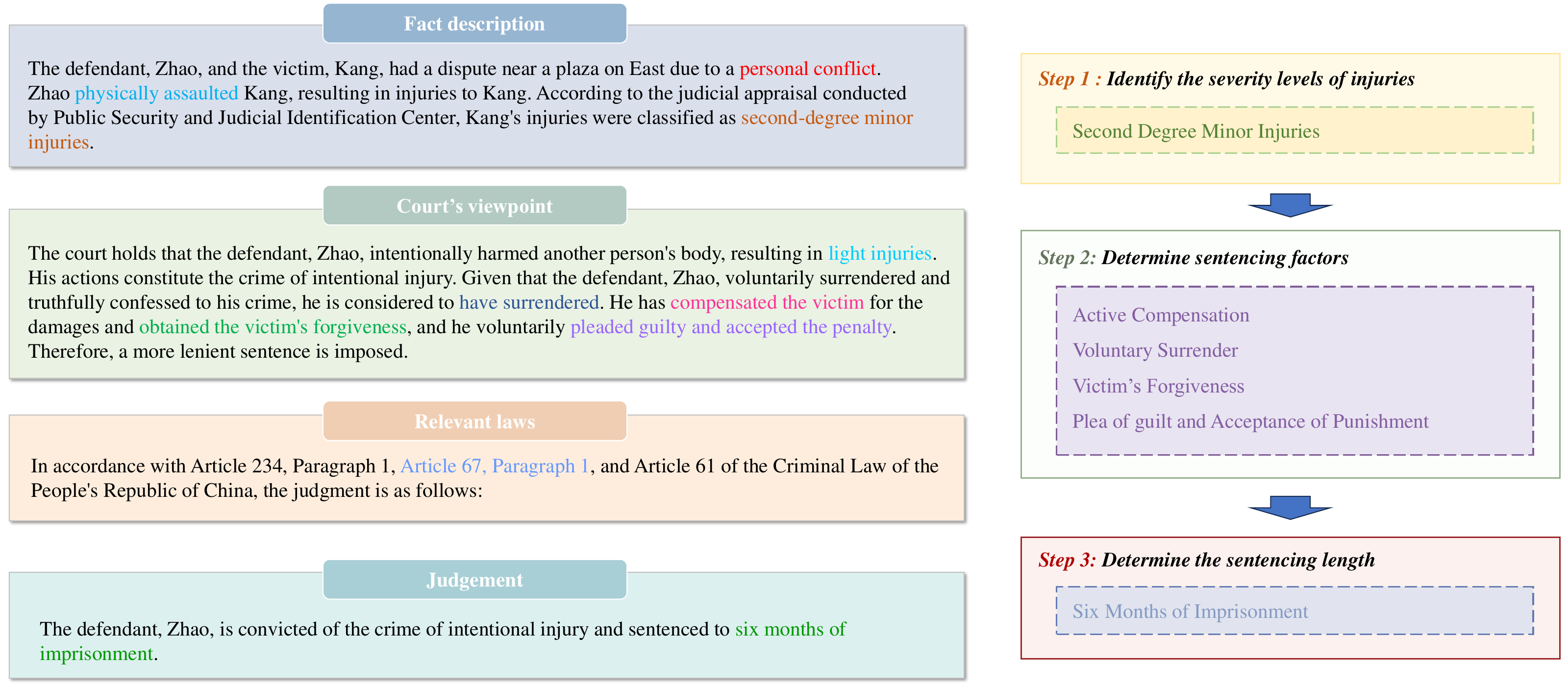}
    \caption{Extract features from different parts of the judgment document}
    \label{fig:-labertfeel11}
\end{figure*}

The CIBH dataset was constructed through a multi-stage process, as illustrated in Figure \ref{fig:-labertfeel11}. Initially, we acquired 4,305 original judgment documents from China Judgments Online\footnote{\url{https://wenshu.court.gov.cn/}} as our primary data source. Subsequently, leveraging relevant empirical studies, we extracted a total of 82 key features from these documents. These features underwent rigorous evaluation by legal professionals, confirming their significant relevance to sentencing prediction in intentional injury cases. The structured extraction of these features was performed using a combination of regular expression matching and large language model (LLM) processing techniques.

To ensure the reliability of our extraction process, we conducted manual verification on a dedicated test set, achieving a precision rate of 98\%. Future research will focus on exploring more sophisticated extraction methodologies to further improve this accuracy. In the resulting CIBH dataset, most features are represented in a binary format (0 or 1), where a value of 1 signifies the presence of the specific behavior or circumstance associated with that feature. For instance, a value of 1 for the "self-defense" feature indicates that the defendant's actions were legally classified as self-defense. Further comprehensive details about the CIBH dataset can be found in Appendix \ref{dataset}.

\subsection{Validation of Theoretical Optimal Accuracy}

Drawing upon the CIBH dataset, we independently modeled the distribution of the noise term $\varepsilon_k$ for both minor and serious injury cases. As depicted in Figures \ref{fig:a} and \ref{fig:b}, the noise associated with minor injuries was found to conform to a normal distribution, specifically $\mathrm{N}(0, 11.70)$. In contrast, the noise distribution for serious injury cases was characterized as $\mathrm{N}(0, 84.13)$. In accordance with Theorem \ref{theorem new}, and predicated on the assumption of known true parameter values, the upper bound of the best achievable $RA^*(T)$ was calculated to be 83.61\% for minor injury cases and 95.13\% for serious injury cases within the CIBH dataset.

\begin{figure}[H]
  \centering
  \begin{minipage}[b]{0.49\linewidth} 
    \includegraphics[width=\linewidth]{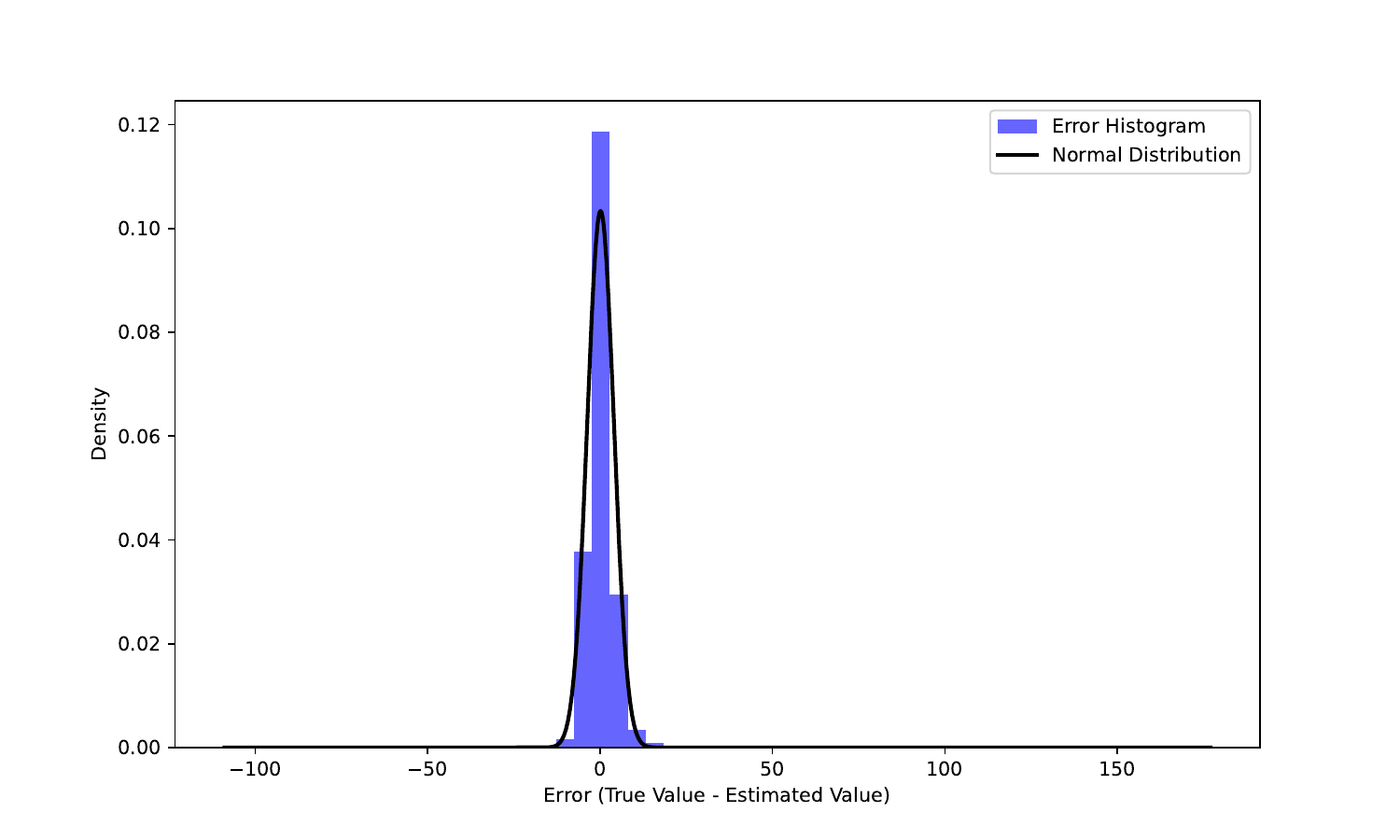}
    \caption{minor variance}
    \label{fig:a}
  \end{minipage}
  \hfill
  \begin{minipage}[b]{0.49\linewidth} 
    \includegraphics[width=\linewidth]{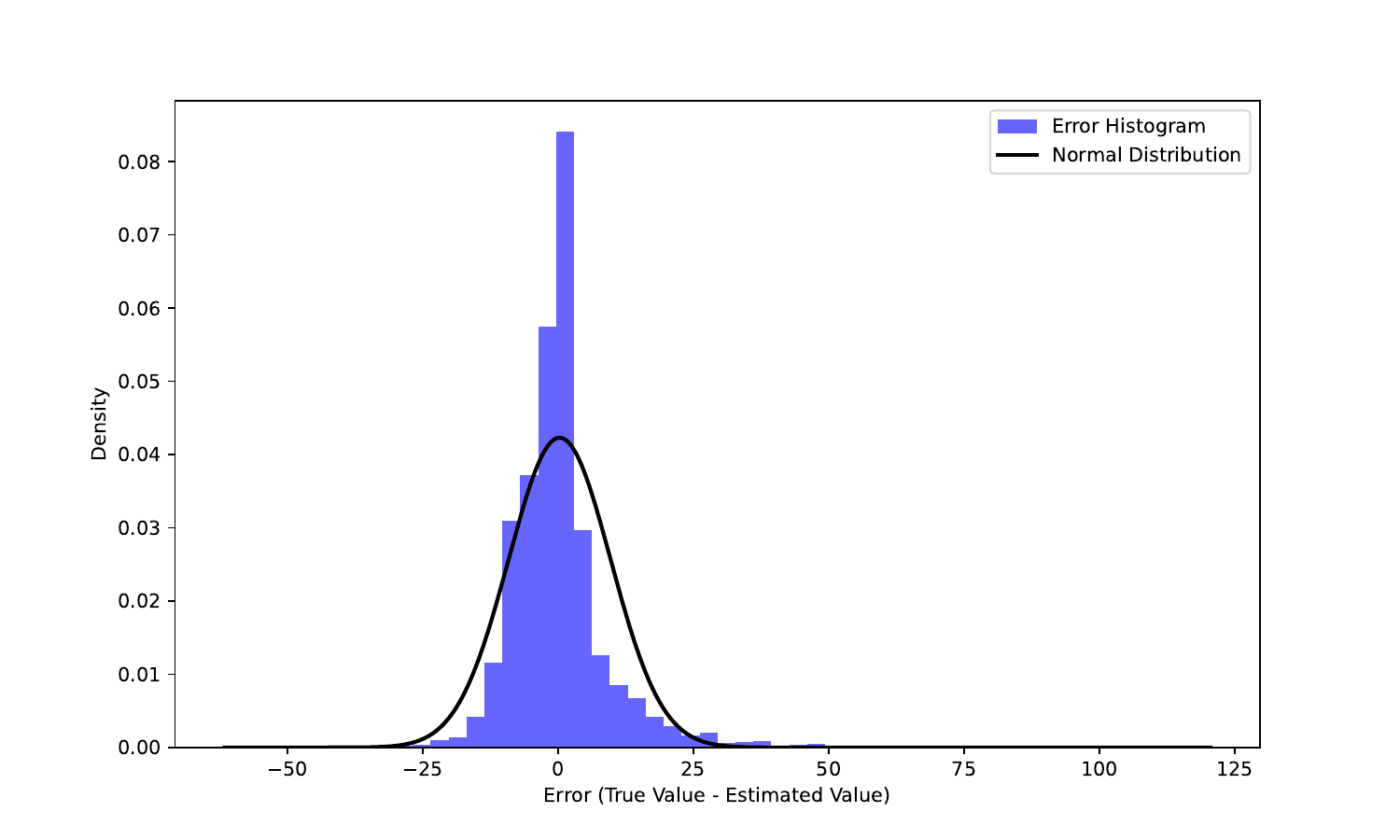}
    \caption{serious variance}
    \label{fig:b}
  \end{minipage}
\end{figure}

\subsection{Effectiveness Evaluation}\label{Effectiveness Evaluation}

\begin{table}[htbp]
\centering
\caption{Performance Comparison of RA(T)}
\label{tab:model_comparison11}
\begin{tabular}{lccccccc}
\toprule
Metric & LR & RF & MLP & LGBM & XGB & CB & SMS \\
\midrule
Serious & 83.25 & 86.83 & 85.67 & 86.41 & 87.05 & 87.36 & \textbf{91.34} \\
Minors & 73.15 & 75.01 & 74.25 & 76.19 & 76.23 & 76.16 & \textbf{77.53} \\
\bottomrule
\end{tabular}
\end{table}

To validate the efficacy of our novel three-stage mechanism model for sentencing prediction, we conducted a comprehensive comparative analysis against state-of-the-art baseline predictors. These baselines included Linear Regression (LR), Random Forest (RF), Multi-Layer Perceptron (MLP), LightGBM (LGBM), XGBoost (XGB), and CatBoost (CB). 

The dataset was stratified into serious and minor injury groups, and within each stratum, it was further partitioned into training and testing sets using an 8:2 ratio. Each baseline model underwent hyperparameter optimization on both injury severity groups via RandomizedSearchCV. The negative mean absolute percentage error served as the scoring metric during this optimization process, aligning with the Relative Accuracy (RA) evaluation metric employed in this study. We randomly sampled 30 parameter combinations from a predefined search space and identified the optimal configuration for each baseline. To rigorously assess the stability and generalization capability of these models, we further employed five random seeds evaluation to estimate error bars, using their
optimal hyperparameter configurations. 
For our proposed MLMS algorithm, the parameter settings were as follows: for the serious injury group, $\sigma = 9.17, \mu = 10, \beta = 0.9$; and for the minor injury group, $\sigma = 3.42, \mu = 1, \beta = 0.5$. For serious injury cases, the saturation upper ($U_k$) and lower ($L_k$) bounds were set at 120 and 36 months, respectively. In minor injury cases, these bounds were 36 and 6 months. All experiments were conducted on a computer equipped with an RTX 3070 GPU. 

As evidenced in Table \ref{tab:model_comparison11}, our proposed model demonstrably outperformed all baseline models, achieving improvements of 3.98\% for serious injuries and 1.3\% for minor injuries, thereby underscoring its superior effectiveness.
Moreover, the adaptive prediction accuracies of our methods for serious and minor injury cases are more close to the theoretical upper bounds of 95.13\% and  83.61\%, respectively. Further details regarding the experiment can be found in Appendix \ref{hyperanbhbu}.

\section{Conclusion} \label{conclusion}

This research addresses the practical hurdles in judicial sentencing prediction, where legal interpretability and high accuracy are paramount despite the nonstationarity and nonindependence of the real judicial data. To ensure interpretability, we introduced a refined Sentencing Mechanism Sentencing (SMS) model rooted in the judicial logic of Chinese Criminal Law. For high adaptive prediction accuracy, we developed an adaptive MLMS algorithm based on the SMS model. To circumvent the restrictive assumptions of stationarity and independence often found in theoretical analyses, our work leverages both Lyapunov function methods and martingale convergence theory, techniques previously applied in the area of stochastic adaptive control. Departing from prevalent empirical studies in judicial sentencing, we rigorously established a theoretical lower bound for our adaptive predictor's accuracy, alongside the optimal upper bound achievable with known parameters. Empirical validation using our constructed CIBH dataset demonstrates the superior performance of our adaptive predictor against standard baselines, achieving an accuracy level not far from the theoretical best. While our study centers on sentencing prediction for intentional bodily harm in China, the proposed modeling and algorithmic framework exhibits broader applicability to other crime types and even diverse application domains. Notably, our methodology for evaluating the theoretical bounds of a given prediction accuracy metric using real-world data holds significant potential for wider adoption.

\small
\bibliographystyle{IEEEtran}
\bibliography{example_paper}

\newpage
\appendix

\section{Proof of Main Theorems} \label{Proof of main theorems}
In this section, we complete the proof of the Theorems. We begin by introducing the following notation.
\begin{equation}
    \psi_{k}=G_k\left(\varphi_k^{\top} \theta_k\right) -G_k\left(\varphi_k^{\top} \hat{\theta}_k\right).
\end{equation}
 
Since $w_{k+1}=y_{k+1}-G_k\left(\varphi_k^{\top} \theta_k\right)$, one can deduce that $\{ w_{k}, \mathcal{F}_k\}$ is a martingale difference sequence with Assumptions \ref{A1}--\ref{A3}. Moreover, $\mathbb{E}_k[|w_{k+1}|^{\alpha}] < \infty$ for any constant $\alpha\geq1$.

We also need the following lemmas.

\begin{lemma}(\cite{cheney2001analysis}) \label{projections}
The projection operator introduced in Definition \ref{projection} satisfies the following property.
\begin{equation}
    \| \pi_D(x) - \pi_{D}(y)\| \leq \|x-y\|, ~\forall x,y \in\mathbb{R}.
\end{equation}
\end{lemma}

\begin{lemma} \label{martingconvergence}(Martingale Convergence Theorem, \cite{chen1991identification})
Let \( \{w_k, \mathcal{F}_k, k\geq0\} \) be a martingale difference sequence and \( \{g_k, \mathcal{F}_k, k \geq0\} \) an adapted process. Suppose there exists \( \alpha \in (0, 2] \) such that
\[
\sup_{k \geq 0} \mathbb{E}\left[|w_{k+1}|^\alpha \mid \mathcal{F}_k\right] < \infty \quad \text{a.s.}
\]
Then for any \( \eta > 0 \), as \( k \to \infty \), the following almost sure bound holds:
\[
\sum_{i=0}^k g_i w_{i+1} = O\left(S_k(\alpha) \cdot \log^{1/\alpha + \eta}(S_k^\alpha(\alpha) + e)\right), \quad \text{a.s.}
\]
where
\[
S_k(\alpha) := \left( \sum_{i=0}^k |g_i|^\alpha \right)^{1/\alpha}.
\]
\end{lemma}

\begin{lemma}\label{lemma2}
If Assumptions \ref{A1}--\ref{A3} hold,
let $\beta_k=\frac{1}{k^\delta}, \delta \in(0,1)$.
The estimation error satisfies the following property:
\begin{equation}
 \left\|\tilde{\theta}_{k+1}\right\|^2 \leq \left\|\tilde{\theta}_k\right\|^2-\mu \underline{g}^2_k \alpha_k   -\frac{2 \mu \varphi_k^{\top} \tilde{\theta}_k G_k^{\prime}\left(\varphi_k^{\top} \hat{\theta}_k\right) m_k}{1+\left\|\varphi_k\right\|^2} w_{k+1}+O\left(\mu w_k^2+\beta_k +\mu\beta_k|w_{k+1}|+\xi_{k+1}\right)   
\end{equation}
where
$$
\tilde{\theta}_k= \theta_k- \hat{\theta}_k, \quad\alpha_k \triangleq \frac{\left\|\varphi_k^{\top} \tilde{\theta}_k\right\|^2}{1+\left\|\varphi_k\right\|^2}, \quad \xi_{k+1} \triangleq\left\|\theta_{k+1}-\theta_k\right\|.
$$
\end{lemma}

\begin{proof}
Building on Lemma \ref{projections}, Algorithm \ref{algorithm1}, and the methodology proposed in \cite{guo2020tima}, we consider
\begin{equation}\label{normtheta}
\begin{aligned} 
& \left\|\theta_k-\hat{\theta}_{k+1}\right\|^2 \\ 
\leq & \| \theta_k-\hat{\theta}_k-\frac{\mu  G_k^{\prime}\left(\varphi_k^{\top} \hat{\theta}_k\right)\varphi_k}{1+\left\|\varphi_k\right\|^2}\left(y_{k+1}-G_k\left(\varphi_k^{\top} \hat{\theta}_k\right)\right)-\beta_k\left(\hat{\theta}_k-\hat{\theta}_{k-1}\right) \|^2 \\ 
=&\left\|\theta_k-\hat{\theta}_k-\frac{\mu  G_k^{\prime}\left(\varphi_k^{\top} \hat{\theta}_k\right)\varphi_k}{1+\left\|\varphi_k\right\|^2}\left( \psi_k+w_{k+1}\right)-\beta_k\left(\hat{\theta}_k-\hat{\theta}_{k-1}\right)\right\|^2 \\ 
=& \left\|\left(I-\frac{\mu  G_k^{\prime}\left(\varphi_k^{\top} \hat{\theta}_k\right) G_k^{\prime}\left(\varphi_k^{\top} \bar{\theta}_k\right)}{1+\left\|\varphi_k\right\|^2}\varphi_k \varphi_k^{\top}\right) \tilde{\theta}_k-\frac{\mu  G_k^{\prime}\left(\varphi_k^{\top} \hat{\theta}_k\right)\varphi_k}{1+\left\|\varphi_k\right\|^2}w_{k+1}-\beta_k \left(\hat{\theta}_k-\hat{\theta}_{k-1}\right) \right\|^2 \\ 
\leq&  \left\|\left(I-\frac{\mu G_k^{\prime}\left(\varphi_k^{\top} \hat{\theta}_k\right) G_k^{\prime}\left(\varphi_k^{\top} \bar{\theta}_k\right)}{1+\left\|\varphi_k\right\|^2}  \varphi_k \varphi_k^{\top}\right) \tilde{\theta}_k-\frac{\mu  G_k^{\prime}\left(\varphi_k^{\top} \hat{\theta}_k\right)\varphi_k}{1+\left\|\varphi_k\right\|^2} w_{k+1}\right\|^2 
\\
&- 2 \beta_k\left[\tilde{\theta}_k^{\top}\left(I-\frac{\mu G_k^{\prime}\left(\varphi_k^{\top} \hat{\theta}_k\right) G_k^{\prime}\left(\varphi_k^{\top} \bar{\theta}_k\right)}{1+\left\|\varphi_k\right\|^2}  \varphi_k \varphi_k^{\top}\right) -\frac{\mu G_k^{\prime}\left(\varphi_k^{\top} \hat{\theta}_k\right) \varphi_k^{\top}}{1+\left\|\varphi_k\right\|^2}w_{k+1}\right]\left(\hat{\theta}_k-\hat{\theta}_{k-1}\right)\\
&+\beta_k^2 {\left\|\hat{\theta}_k-\hat{\theta}_{k-1}\right\|}^2,
\end{aligned}
\end{equation}
where we use the fact that $\psi_k = G_k^{\prime}\left(\varphi_k^{\top} \bar{\theta}_k\right)\varphi_k^{\top}\tilde{\theta}_k$, with $\varphi_k^{\top} \bar{\theta}_k$ lying between $\varphi_k^{\top}\theta_k$ and $\varphi_k^{\top}\hat{\theta}_k$ by the Mean Value Theorem.

Let us  analyze  the RHS of \eqref{normtheta} term by term.

For the  the first term, we can deduce that 
\begin{equation}\label{saturated important}
\begin{aligned}
&  \left\|\left(I-\frac{\mu G_k^{\prime}\left(\varphi_k^{\top} \hat{\theta}_k\right) G_k^{\prime}\left(\varphi_k^{\top} \bar{\theta}_k\right)}{1+\left\|\varphi_k\right\|^2}  \varphi_k \varphi_k^{\top}\right) \tilde{\theta}_k-\frac{\mu  G_k^{\prime}\left(\varphi_k^{\top} \hat{\theta}_k\right)\varphi_k}{1+\left\|\varphi_k\right\|^2} w_{k+1}\right\|^2\\
=& \left\|\tilde{\theta}_k\right\|^2-2 \frac{\mu G_k^{\prime}\left(\varphi_k^{\top} \hat{\theta}_k\right) G_k^{\prime}\left(\varphi_k^{\top} \bar{\theta}_k\right)\left(\varphi_k^{\top} \tilde{\theta}_k\right)^2}{1+\left\|\varphi_k\right\|^2}\\
&+\frac{\mu^2\left\|\varphi_k\right\|^2\left(G_k^{\prime}\left(\varphi_k^{\top} \hat{\theta}_k\right) G_k^{\prime}\left(\varphi_k^{\top} \bar{\theta}_k\right)\right)^2\left(\varphi_k^{\top} \tilde{\theta}_k\right)^2}{\left(1+\left\|\varphi_k\right\|^2\right)^2} \\
& -2 \tilde{\theta}_k^{\top}\left(1-\frac{\mu  G_k^{\prime}\left(\varphi_k^{\top} \hat{\theta}_k\right) G_k^{\prime}\left(\varphi_k^{\top} \bar{\theta}_k\right)}{1+\left\|\varphi_k\right\|^2}\varphi_k \varphi_k^{\top}\right) \frac{\mu G^{\prime} \left(\varphi_k^{\top} \hat{\theta}_k\right)\varphi_k}{1+\left\|\varphi_k\right\|^2} w_{k+1} \\
& +\frac{\mu^2 \|\varphi_k\|^2\left(G_k^{\prime}\left(\varphi_k^{T} \hat{\theta}_k\right)\right)^2}{\left(1+\left\|\varphi_k\right\|^2\right)^2}w_{k+1}^2 \\
\leq& \left\|\tilde{\theta}_k\right\|^2-\frac{\mu G_k^{\prime}\left(\varphi_k^{\top} \hat{\theta}_k\right) G_k^{\prime}\left(\varphi_k^{\top} \bar{\theta}_k\right)\left(\varphi_k^{\top} \tilde{\theta}_k\right)^2}{1+\left\|\varphi_k\right\|^2} \\
&  -\left[1-\frac{\mu \| \varphi_k\|^2 G^{\prime}\left(\varphi_k^{\top} \hat{\theta}_k\right) G^{\prime}\left(\varphi_k^{\top} \bar{\theta}_k\right)}{1+\left\|\varphi_k\right\|^2}\right]\frac{2 \mu \varphi_k^{\top} \tilde{\theta}_k G_k^{\prime}\left(\varphi_k^{\top} \hat{\theta}_k\right)}{1+\left\|\varphi_k\right\|^2}w_{k+1} \\
& +\frac{\mu^2}{1+\left\|\varphi_k\right\|^2}w_{k+1}^2 \\
& \leq \left\|\tilde{\theta}_k\right\|^2-\mu \underline{g}^2_k \alpha_k   -\frac{2 \mu \varphi_k^{\top} \tilde{\theta}_k G_k^{\prime}\left(\varphi_k^{\top} \hat{\theta}_k\right) m_k}{1+\left\|\varphi_k\right\|^2} w_{k+1}+\mu^2 w^2_{k+1},\\
\end{aligned}
\end{equation}

where $m_k = 1-\frac{\mu \| \varphi_k\|^2 G^{\prime}\left(\varphi_k^{\top} \hat{\theta}_k\right) G^{\prime}\left(\varphi_k^{\top} \bar{\theta}_k\right)}{1+\left\|\varphi_k\right\|^2}$. The first inequality holds because $\frac{\|\varphi_k\|^2}{1+\|\varphi_k\|^2} \leq 1$ and $0 < \mu \leq 1$. Moreover, since $G_k^{\prime}(\cdot) \leq 1$ almost surely, it follows that $0 \leq m_k \leq 1$ almost surely.

As for the second term on the RHS of \eqref{normtheta}, we have
\begin{equation} \label{sencondterm2}
\begin{aligned} 
&\beta_k\left[\tilde{\theta}_k^{\top}\left(I-\frac{\mu G_k^{\prime}\left(\varphi_k^{\top} \hat{\theta}_k\right) G_k^{\prime}\left(\varphi_k^{\top} \bar{\theta}_k\right)}{1+\left\|\varphi_k\right\|^2}  \varphi_k \varphi_k^{\top}\right) -\frac{\mu G_k^{\prime}\left(\varphi_k^{\top} \hat{\theta}_k\right) \varphi_k^{\top}}{1+\left\|\varphi_k\right\|^2}w_{k+1}\right]\left(\hat{\theta}_k-\hat{\theta}_{k-1}\right) \\ 
& =\beta_k \tilde{\theta}_k^{\top}\left(\hat{\theta}_k-\hat{\theta}_{k-1}\right)-\mu \beta_k \frac{\varphi_k^{\top} \tilde{\theta}_k  G_k^{\prime}\left(\varphi_k^{\top} \theta_k\right) G_k^{\prime}\left(\varphi_k^{\top} \tilde{\theta}_k\right)}{1+\left\|\varphi_k\right\|^2}\varphi_k^{\top}\left( \hat{\theta}_k-\hat{\theta}_{k-1}\right) \\
&-\mu \beta_k \frac{G_k^{\prime}\left(\varphi_k^{\top} \hat{\theta}_k\right)}{1+\left\|\varphi_k\right\|^2} \varphi_k^{\top}\left(\hat{\theta}_k-\hat{\theta}_{k-1}\right) w_{k+1}.
\end{aligned}
\end{equation}

Since both $\theta_{k}$ and $\hat{\theta}_{k}$ belong to $\mathcal{D}$, each component satisfies:
\begin{equation}
    \left|\theta_{k, i}\right| \leq L \quad \text { and } \quad\left|\hat{\theta}_{k, i}\right| \leq L.
\end{equation}
Then one can deduce that
\begin{equation} \label{bounded_theta2}
    \left\|\tilde{\theta}_k\right\| \leq 2 p^{1 / 2} L,  ~~\|\hat{\theta}_{k+1} - \hat{\theta}_{k}\|\leq 2 p^{1 / 2} L.
\end{equation}

Therefore, according to \eqref{sencondterm2}, \eqref{bounded_theta2} and $ 0<\mu\leq1$, we know that
\begin{equation}\label{crossbound}
    \begin{aligned}
        &\beta_k\left[\tilde{\theta}_k^{\top}\left(I-\frac{\mu G_k^{\prime}\left(\varphi_k^{\top} \hat{\theta}_k\right) G_k^{\prime}\left(\varphi_k^{\top} \bar{\theta}_k\right)}{1+\left\|\varphi_k\right\|^2}  \varphi_k \varphi_k^{\top}\right) -\frac{\mu G_k^{\prime}\left(\varphi_k^{\top} \hat{\theta}_k\right) \varphi_k^{\top}}{1+\left\|\varphi_k\right\|^2}w_{k+1}\right]\left(\hat{\theta}_k-\hat{\theta}_{k-1}\right) \\
        =&O(\beta_k + \mu\beta_k|w_{k+1}|)
    \end{aligned}
\end{equation}

With \eqref{normtheta}, \eqref{saturated important} and \eqref{crossbound}, one can deduce that  the \textbf{Lyapunov function} $V_{k+1}=\left\|\tilde{\theta}_{k+1}\right\|^2$ satisfies
\begin{equation}
    \begin{aligned}
 \left\|\tilde{\theta}_{k+1}\right\|^2=&\left\|\theta_{k+1}-\hat{\theta}_{k+1}\right\|^2 \\
= & \left\|\left(\theta_k-\hat{\theta}_{k+1}\right)+\left(\theta_{k+1}-\theta_k\right)\right\|^2 \\
= & \left\|\theta_k-\hat{\theta}_{k+1}\right\|^2+O\left(\xi_{k+1}\right)+\xi_{k+1}^2 \\
\leq & \left\|\tilde{\theta}_k\right\|^2-\mu\underline{g}^2_k  \alpha_k  -\frac{2 \mu \varphi_k^{\top} \tilde{\theta}_k G_k^{\prime}\left(\varphi_k^{\top} \hat{\theta}_k\right) m_k}{1+\left\|\varphi_k\right\|^2} w_{k+1}\\
&+O\left(\mu^2 w_{k+1}^2+\beta_k +\mu\beta_k|w_{k+1}|+\xi_{k+1}\right).
\end{aligned}
\end{equation}
The final inequality holds because $\xi$ is sufficiently small.

This completes the proof.
\end{proof}

Now we are ready to prove the main theorem.
\begin{proof}[Proof of Theorem \ref{main theorem}] It follows from Lemma \ref{lemma2} that

\begin{equation}\label{inequal}
\begin{aligned}
    &\sum_{k=0}^T \alpha_k + \sum_{k=0}^T \frac{2 \mu \varphi_k^{\top} \tilde{\theta}_kG_k^{\prime}\left(\varphi_k^{\top} \hat{\theta}_k\right) m_k}{(1+\left\|\varphi_k\right\|^2)h^2} w_{k+1}\\
    \leq& \frac{1}{\mu h^2}\left[\|\tilde{\theta}_0\|^2 - \|\tilde{\theta}_{k+1}\|^2\right]\\
&+ O\left(\sum_{k=0}^T \mu^2w_{k+1}^2 + \sum_{k=0}^T \beta_k + \sum_{k=0}^T \mu\beta_k |w_{k+1}|+ \sum_{k=0}^T \xi_{k+1}\right),
\end{aligned}
\end{equation}
where $h = \inf\limits_{ k \geq 0, ~|x| \leq C} G_k^{\prime}(x)$.

With Lemma \ref{martingconvergence} and noticing the boundedness of $h$, $\mu$, $m_k$, we have
\begin{equation}
     \sum_{k=0}^T \frac{2 \mu \varphi_k^{\top} \tilde{\theta}_kG_k^{\prime}\left(\varphi_k^{\top} \hat{\theta}_k\right) m_k}{(1+\left\|\varphi_k\right\|^2)h^2} w_{k+1}
     = o(\sum_{k=0}^{T}\alpha_k), \quad \text{a.s.}
\end{equation}

With Lemma \ref{martingconvergence} and noticing that $\mathbb{E}[w_{k+1}^4|\mathcal{F}_k] \leqslant \infty$, we have
\begin{equation}\label{thm1}
\begin{aligned}
    \sum_{k=0}^T w_{k+1}^2 =& \sum_{k=0}^T \left[w_{k+1}^2 - \mathbb{E}[w_{k+1}^2 \mid \mathcal{F}_{k}]\right] + \sum_{k=0}^T \mathbb{E}[w_{k+1}^2 \mid \mathcal{F}_{k}]\\
    &\leq \sum_{k=0}^T \mathbb{E}[w_{k+1}^2 \mid \mathcal{F}_{k}]+o(T), ~\text{a.s.}
\end{aligned}
\end{equation}
Where we use the fact that $\{ w_{k+1}^2 - \mathbb{E}[w_{k+1}^2 \mid \mathcal{F}_{k}]\} $ is a martingale difference sequence.

Since $\beta_k=\frac{1}{k^\delta}$ and $\delta \in(0,1)$, $\sum\limits_{k=1}^T \frac{1}{k^\delta} \leq 1+\int_1^T \frac{1}{x^\delta} d x=1+\frac{T^{1-\delta}-1}{1-\delta} \leq \frac{2}{1-\delta} T^{1-\delta}$, then we have
\begin{equation}\label{thm2}
    \sum\limits_{k=0}^T \beta_k = o(T).
\end{equation}

Moreover, according to lemma \ref{martingconvergence} and following the analysis of \eqref{thm1} and \eqref{thm2}, one can deduce that
\begin{equation} \label{thm3}
    \sum_{k=0}^T \beta_k |w_{k+1}| = o(T), ~\text{a.s.}
\end{equation}

Combing \eqref{thm1}, \eqref{thm2} and \eqref{thm3}, we have
\begin{equation}
    \sum_{k=0}^T \alpha_k \leqslant \frac{2}{\mu h^2 }\left[\|\tilde{\theta}_0\|^2 - 2\|\tilde{\theta}_{k+1}\|^2\right] + \mu^2\sum_{k=0}^T \mathbb{E}[w_{k+1}^2 \mid \mathcal{F}_{k}] + O(\sum_{k=0}^T \xi_{k+1}) + o(T), ~\text{a.s.}
\end{equation}

Combing  this, we have

\begin{equation}\label{desired inequality1}
   \begin{aligned}
& \lim _{T \rightarrow \infty} \sup \frac{1}{T} \sum_{k=0}^T (y_{k+1}-\hat{y}_{k+1})^2 \\
=& \lim _{T \rightarrow \infty} \sup \frac{1}{T} \sum_{k=0}^T (\psi_{k} + w_{k+1})^2 \\
\leq & \limsup _{T \rightarrow \infty} \frac{1}{T} \sum_{k=0}^T\left[\frac{1}{h^2}\alpha_k^2(1+\|\varphi_k\|^2) + 2  G_k^{\prime}\left(\varphi_k^{\top} \bar{\theta}_k \right)(\varphi_k^{\top}\tilde{\theta}_k)w_{k+1}+ w_{k+1}^2\right]\\
\leq&  \sigma_w^2+\limsup _{T \rightarrow \infty} \left\{ \frac{1}{Th^2} \sum_{k=0}^T \alpha_k\left(1+\left\|\varphi_k\right\|^2\right) \right\} \\
\leq& (1+\mu^2)\sigma_w^2 + O(\xi), ~ \text{a.s.}
\end{aligned} 
\end{equation}
where $G_k^{\prime}\left(\varphi_k^{\top} \bar{\theta}_k \right)$ defined in \eqref{normtheta} and $h$ is defined in \eqref{inequal}. The second equality holds due to the fact that $\sum\limits_{k=0}^{n}G_k^{\prime}\left(\varphi_k^{\top} \bar{\theta}_k \right)(\varphi_k^{\top}\tilde{\theta}_k)w_{k+1} = o(\sum\limits_{k=0}^{n}\alpha_k(1+\|\varphi_k\|^2))$ due to Lemma \ref{martingconvergence} and the boundedness of $G_k^{\prime}(\cdot)$ and $\|\varphi_k\|$. The last inequality holds due to the fact \eqref{bountheta} and the Lyapunov inequality.

From \eqref{desired inequality1} and the Cauchy–Schwarz inequality, it follows that
\begin{equation}
\begin{aligned}
    &\limsup _{T \rightarrow \infty} \frac{1}{T} \sum_{k=1}^T\frac{|y_{k+1}-\hat{y}_{k+1}|}{y_{k+1}} \\
    \leq& \limsup _{T \rightarrow \infty} \frac{1}{T} \left(\sum_{k=1}^T \frac{1}{y_k^2}\right)^{\frac{1}{2}}\left(\sum_{k=1}^T (y_{k+1}-\hat{y}_{k+1})^2\right)^{\frac{1}{2}} \\
    \leq & \frac{1}{c} \limsup _{T \rightarrow \infty} \left(\sum_{k=1}^T \frac{1}{y_k^2}\right)^{\frac{1}{2}}\left(\frac{1}{T}\sum_{k=1}^T (y_{k+1}-\hat{y}_{k+1})^2\right)^{\frac{1}{2}}\\
    =&O\left(\frac{\sigma_w}{c}\right)+O\left(\frac{\xi^{\frac{1}{2}}}{c}\right),~ \text{a.s.}
\end{aligned}
\end{equation}
where $c>0$ is the constant defined in Theorem \ref{main theorem}. This completes the proof.

\end{proof}

Intuitively, a RA value closer to 1 indicates higher precision in sentencing prediction. However, when sentence lengths are affected by noise, the expected absolute error $\mathbb{E}|y_k - \hat{y}_k|$ is bounded below by the noise variance, making it impossible for the Relative Accuracy to reach 1 as $n$ tends infinity. To illustrate this point, we introduce the following lemmas.

\begin{lemma} (See, e.g., \cite{guo2020tima}, Theorem 1.2.18.) \label{von neumann}
Let \(\{a_k, k \geq 0\}\) be a bounded non-negative sequence. A necessary and sufficient condition for the \(\lim\limits_{n \to \infty} \frac{1}{n} \sum\limits_{k=0}^{n-1} a_k = 0\) to hold is that there exists a set \(E \subset \mathbb{Z}_{+}\) with density zero (i.e., \(\lim\limits_{n \to \infty} \frac{1}{n} \sum\limits_{k=0}^{n-1} I(k \in E) = 0\)), such that \(\lim\limits_{\substack{n \to \infty \\ \ n \notin E}} a_n = 0\).
\end{lemma}

We now provide a rigorous clarification that the RA cannot attain 1.
\begin{proposition}
\label{Claim 1}
Suppose Assumptions \ref{A1}, \ref{A2}, and \ref{bouned noise} hold. 
Then there exists a constant \( 0 \leq c_3 < 1 \) such that, for any predictor \( \hat{y}_k \in [L_k, U_k]\) , 
the following upper bound on the asymptotic relative error holds:
\begin{equation}
\limsup_{T \rightarrow \infty} \left(1 - \frac{1}{T} \sum_{k=1}^T \frac{\left|y_k - \hat{y}_k\right|}{y_k} \right) \leq c_3, ~ \text{a.s.}.
\end{equation}
\end{proposition}
\begin{proof}[Proof of Proposition \ref{Claim 1}]
    We prove this proposition by contradiction. Suppose that there exists a sequence $\{\hat{y}_{k_{i}}\}$ where $\{k_{i}\}$ is a subsequence of $\{k\}$ such that $ \lim\limits_{T \rightarrow \infty} 1 - \frac{1}{T} \sum_{i=1}^T \frac{\left|y_{k_i} - \hat{y}_{k_i}\right|}{y_{k_i}} = 1$. Without loss of generality, we assume that the sequence  $\{k_{i}\}$ coincides with  $\{k\}$ for convenience of analysis.
    This implies that
    \begin{equation}
        \lim_{T \rightarrow \infty} \frac{1}{T} \sum_{k=1}^T \frac{\left|y_k - \hat{y}_k\right|}{y_k} = 0.
    \end{equation}
According to Lemma \ref{von neumann}, one can know that there exists a zero-density set $E\in \mathbb{Z}+$ such that
\begin{equation}
    \lim\limits_{\substack{k \to \infty \\ \ k \notin E}} \frac{\left|y_k - \hat{y}_k\right|}{y_k} = 0,
\end{equation}
since $y_k > c$, $\forall k \geq 1$, we have
\begin{equation}\label{tends0}
    \lim\limits_{\substack{k \to \infty \\ \ k \notin E}} \left|y_k - \hat{y}_k\right| = 0.
\end{equation}
Taking conditional expectations on both sides, we obtain
\begin{equation}
    \begin{aligned}
         \lim\limits_{\substack{k \to \infty \\ \ k \notin E}} \mathbb{E}_{k-1}\left|y_k - \hat{y}_k\right| = 0.
    \end{aligned}
\end{equation}
However, under Assumptions \ref{A1}, \ref{A2}, and \ref{bouned noise}, the conditional expectation is given by
\begin{equation}
    \begin{aligned}
        &\mathbb{E}_{k-1}\left[\left|y_k - \hat{y}_k\right|\right]\\
        =& (U_{k-1} - \hat{y}_k) \left[1 - F_k\left(U_{k-1}  - \varphi_{k-1} ^{\top}\theta_{k-1} \right)\right] + (\hat{y}_k - L_{k-1} )F_k\left(L_{k-1}  -  \varphi_{k-1} ^{\top}\theta_{k-1} \right) \\
        &+ \int_{L_{k-1} -   \varphi_{k-1} ^{\top}\theta_{k-1} }^{U_{k-1}  -   \varphi_{k-1} ^{\top}\theta_{k-1} } |S_k(  \varphi_{k-1} ^{\top}\theta_{k-1} +x) -\hat{y}_k|f_{k}(x) dx \\
        \geq& c_4 >0,
    \end{aligned}
\end{equation}
where $F_k(\cdot)$ denotes the conditional distribution function of $\varepsilon_k$ given $\mathcal{F}_{k-1}$, $c_4$ is a positive constant. And we also use the fact that $\varphi_k^{\top}\theta_k$, $L_k$ and $U_k$ are bounded uniformly. This leads to a contradiction with \eqref{tends0}, thereby completing the proof.
\end{proof}

To prove Theorem \ref{theorem new}, we need the following definition and lemma.
\begin{definition} (\cite{zheng2024l1})\label{mediand}
Consider a non-decreasing sequence of $\sigma$-algebras $\{\mathcal{F}_k\}_{k \geq 0}$, and let $\{x_k\}_{k \geq 0}$ be an $\mathcal{F}_k$-adapted stochastic process. The conditional distribution of the random variable $x_{k+1}$ given $\mathcal{F}_k$ is denoted by $\mathbb{P}_{k+1}(\cdot)$. The corresponding conditional median, $m_k$, which is $\mathcal{F}_k$-measurable, is defined such that $\mathbb{P}_{k+1}(x_{k+1} \leq m_k) \geq 1/2$ and $\mathbb{P}_{k+1}(x_{k+1} \geq m_k) \geq 1/2$.
\end{definition}

We now present a formal theorem that characterizes the upper bound of the RA metric achievable under optimal prediction, i.e., when the true model parameters are known.
\begin{lemma}\label{Claim 2}
    Under the conditions of Theorem~\ref{main theorem} and Assumption~\ref{bouned noise},the optimal adaptive predictor that minimizes $\mathbb{E}_k[|y_{k+1}-\hat{y}_{k+1}|]$ is given by   $\hat{y}_{k+1}^{*} = S_k(\varphi_{k}^{\top}\theta_{k})$. Moreover, 
    \begin{equation}
        \liminf_{T\rightarrow \infty} \frac{1}{T}\sum_{k=1}^{T}|y_{k+1} - \hat{y}_{k+1}^{*}|  = \liminf_{T\rightarrow \infty}\frac{1}{T}\sum_{k=1}^{T}\sigma_{sk}, ~\text{a.s.}
    \end{equation}
    where $\sigma_{k} = \mathbb{E}_{k}\left[|S_k(\varphi_{k}^{\top}\theta_{k} + \varepsilon_{k+1})-S_k(\varphi_{k}^{\top}\theta_{k} )|\right]$, and $\forall \epsilon >0$.
\end{lemma}

\begin{proof}[Proof of Lemma \ref{Claim 2}]
    According to Lemma 1 in \cite{zheng2024l1}, we can deduce that $\hat{y}_{k+1}^{*} = S_k(\varphi_{k}^{\top}\theta_{k})$. Moreover,
    \begin{equation}
        \begin{aligned}
             &\frac{1}{T}\sum_{k=1}^{T}|y_{k+1} - \hat{y}_{k+1}^{*}|\\
             =& \frac{1}{T}\sum_{k=1}^{T} \left[|y_{k+1} - \hat{y}_{k+1}^{*}| - \sigma_{k}\right] + \frac{1}{T}\sum_{k=1}^{T}\sigma_{k}.
        \end{aligned}
    \end{equation}
According to Lemma \ref{martingconvergence}, we know that $ \frac{1}{T}\sum_{k=1}^{T} \left[|y_{k+1} - \hat{y}_{k+1}^{*}| - \sigma_{k}\right] = o(1)$ since $\{ |y_{k+1} - \hat{y}_{k+1}^{*}| - \sigma_{k}\}$ is a martingale difference sequence and the sequence $\{y_k\}$ is bounded. This completes the proof.
\end{proof}

\begin{proof}[Proof of Theorem \ref{theorem new}]
   Based on Lemma \ref{Claim 2}, a best upper bound for $RA(T)$ can be derived as follows:
\begin{equation}
    \limsup_{T \rightarrow \infty} ~~1 - \frac{1}{T} \sum_{k=1}^T \frac{\left|y_k - \hat{y}_k\right|}{y_k} = \limsup_{T \rightarrow \infty} ~~1 - \frac{1}{T}  \sum_{k=1}^T \frac{\sigma_{k}}{y_k} , ~\text{a.s.}
\end{equation}
where the denominator term $y_{k+1}$ in $RA(T)$ is regarded as a known and fixed positive weighting coefficient.
\end{proof}

Furthermore, we provide an example to illustrate how $\sigma_{sk}$ can be calculated if the conditional distribution of the noise is known. 

\textbf{Example}: Consider that the noise sequence $\left\{\varepsilon_k\right\}$ is independent and normally distributed with $\varepsilon_k \sim N\left(0, \sigma^2\right)$. Let the conditional probability distribution function and the conditional probability density function of $\varepsilon_k$ be $F(\cdot)$ and $f(\cdot)$, respectively. Then the function $\sigma_{s(k+1)}(\cdot)$ can be calculated as follows:

\begin{equation}
\begin{aligned}
\sigma_{s(k+1)}(x)= & (U_k - S_k(\varphi_k^{\top}\theta_k))[1-F(U_k - \varphi_k^{\top}\theta_k)]+ (S_k(\varphi_k^{\top}\theta_k)-L_k)F(L_k - \varphi_k^{\top}\theta_k)\\
& +\int_{L_k - \varphi_k^{\top}\theta_k}^{U_k - \varphi_k^{\top}\theta_k} |S_k(\varphi_k^{\top}\theta_k- \varphi_k^{\top}\theta_k-x|f(x) dx\\
=& (U_k - S_k(\varphi_k^{\top}\theta_k))[1-F(U_k - \varphi_k^{\top}\theta_k)]+ (S_k(\varphi_k^{\top}\theta_k)-L_k)F(L_k - \varphi_k^{\top}\theta_k)\\
& +[S_k(\varphi_k^{\top}\theta_k)-\varphi_k^{\top}\theta_k] [2F(S_k(\varphi_k^{\top}\theta_k)-\varphi_k^{\top}\theta_k))-F(U_k - \varphi_k^{\top}\theta_k)-)-F(L_k - \varphi_k^{\top}\theta_k)]\\
& + \sigma^2[2f(S_k(\varphi_k^{\top}\theta_k)-\varphi_k^{\top}\theta_k))-f(U_k - \varphi_k^{\top}\theta_k)-)-f(L_k - \varphi_k^{\top}\theta_k)].
\end{aligned}
\end{equation}

\section{Dataset Construction} \label{dataset}

\subsection{Motivation}
Most existing datasets have insufficient cases for specific offenses, limiting their ability to support sentencing prediction for individual crimes. Datasets such as CAIL \cite{xiao2018cail2018} and LECC \cite{xueleec} offer valuable repositories of raw judicial documents covering various criminal charges. However, they exhibit two significant shortcomings:
(1) Limited Case Coverage per Crime: The number of cases available for specific offenses is often insufficient, constraining models' ability to develop robust generalization for single-crime sentencing prediction \cite{su2025judge}.
(2) Lack of Fine-Grained Sentencing Features: Most datasets fail to include critical legal factors that heavily influence sentencing outcomes, rendering them incompatible with the mechanism sentencing models discussed earlier \cite{liebman2020mass,kuccuk2025computational}. For example, the Chinese legal event detection dataset by Yao et al. \cite{yao2022leven} excludes key sentencing determinants such as voluntary surrender, active compensation, and victim forgiveness, all of which significantly affect sentence severity.
These deficiencies highlight the need for specialized datasets tailored to single-crime sentencing prediction. Such datasets should incorporate comprehensive, legally relevant features to improve the accuracy and reliability of sentencing outcome predictions.

\subsection{Judgment Document Analysis}

The judgment document is structured into four key sections, as shown on the left side of Figure \ref{fig:-labertfeel11}: fact description, court's viewpoint, relevant laws and judgment. Together, these parts provide a clear, reasoned explanation of the court's decision, linking the facts of the case to the applicable laws and ensuring transparency in the judicial process.

The right side of Figure \ref{fig:-labertfeel11} illustrates the process of determining the features required for sentencing based on the judgment document, which naturally form the basis for the prediction of the sentence length. First, we extract the victim's injury severity from the fact description section, such as "second-degree minor injury" in this case. Second, case-specific features are extracted from the fact description, court's viewpoint, and relevant laws. It is important to note that since this paper focuses solely on sentencing type classification and length prediction, and does not consider legal article recommendation task, we use the relevant laws to aid feature extraction. For instance, in this case, the relevant legal article is Article 67, Paragraph 1 of the Criminal Law of the People's Republic of China, which stipulates the factor of the defendant's voluntary surrender. Based on this article, we infer the presence of this mitigating factor in the case. Third, the judgment section is used to extract the sentencing outcome, which serves as the label for sentence length prediction tasks.

While existing datasets \cite{ma2021lecard,huang2024cmdl,xian2024dlee,yao2023unsupervised}provide valuable information by including factual descriptions, relevant legal provisions, and judgment outcomes, the effect of the admissibility of evidence and facts in actual court proceedings should not be neglected \cite{cui2023survey}, and the judgment documents will reflect these information.
Therefore, our dataset integrates both the court's evaluative perspectives and the actual legitimate intent, addressing the issues where the rules of proof are insufficient and evidence remains ambiguous, which are common in most existing datasets.

\subsection{Feature Extraction}

The extracted features are meticulously designed to align with established judicial sentencing logic and exhibit a high degree of congruence with judicial professional standards. This principled feature engineering process ensures that the model's inputs are not merely statistical correlations but are rooted in substantive legal considerations. Specifically, these features are systematically derived from two primary sources: (a) codified Chinese legal statutes, ensuring direct relevance to the legal framework governing sentencing, and (b) empirically validated sentencing considerations prevalent in judicial practice, capturing nuances and factors that routinely influence sentencing decisions beyond the strict letter of the law. This dual-source approach aims to create a feature set that is both legally sound and practically meaningful for the task of judicial sentencing prediction.

Current methods typically employ word embedding techniques to extract key features. However, this approach can lead to incomplete keyword extraction and inaccurate feature identification, especially when the descriptions in judgment documents are ambiguous or imprecise. For instance, a judgment might state: Regarding the defendant's actions mentioned by the defender as self-defense, this court does not accept them. In such cases, the model may correctly identify the phrase "self-defense" but fail to capture the subsequent negation, resulting in erroneous feature extraction. 

To address these challenges, we propose a hybrid method combining regular expression (Regex) matching and large language models (LLM) to extract sentencing features from judicial documents. Traditional regex methods excel at efficiently handling structured or semi-structured data. Since judgment documents typically follow a standardized format, regex matching enhances efficiency when extracting fixed-form features, such as case numbers, dates, and predefined legal clauses. However, regex methods encounter limitations when addressing more complex tasks, such as identifying the severity of harm in legal cases. Conversely, LLMs offer strong semantic understanding and adaptability to various legal texts, enabling the extraction of context-dependent features. Nevertheless, they may produce errors when extracting features that are difficult to describe linguistically, and their high computational costs along with slower processing speeds pose significant challenges. Therefore, by integrating regex and LLMs, our hybrid approach leverages the strengths of both methods while mitigating their respective  drawbacks, resulting in more accurate and efficient feature extraction from large-scale judicial documents. Regarding feature quality assessment, this dataset was constructed through multiple rounds of cross-validation by legal professionals: initial extraction was performed using optimized regular expressions and LLM prompts, followed by manual verification by expert teams. The latest sampling test shows the current feature accuracy exceeds 98\%. We will continue to refine the feature extraction methodology with the accuracy rate expected to further improvement.

\subsection{Feature Analysis}

\begin{figure}[H]
    \centering
    \includegraphics[width=0.8\textwidth]{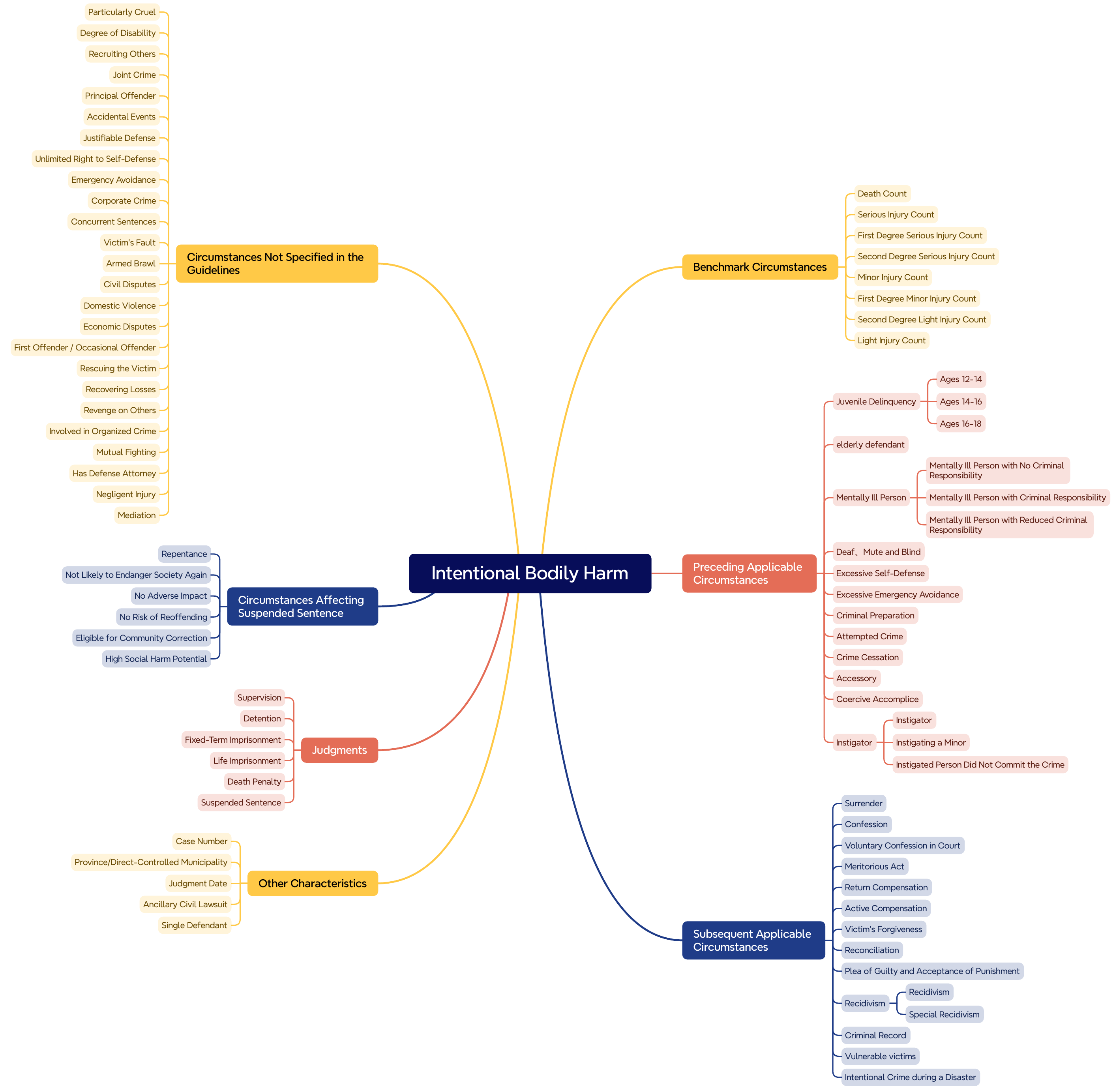}  
    \caption{Sentencing Factors for Intentional Bodily Harm Cases}
    \label{fig:eps_example233}
\end{figure}
As shown in Figure \ref{fig:eps_example233}, we present 82 distinct features relevant to sentencing in intentional bodily harm cases, organized into seven major categories based on the Criminal Law of the People's Republic of China and Sentencing Guidelines for Common Crimes.

These categories encompass a broad spectrum of considerations crucial for assessing the severity of an intentional bodily harm case, spanning not only injury-specific indicators such as the extent of physical harm but also judicial factors such as aggravating or mitigating circumstances. By systematically accounting for these diverse situational attributes, the framework provides a comprehensive basis for evaluating each case, thereby enhancing the clarity and consistency of legal decision-making.

To gain deeper insight into the relationships between sentencing factors, we generated a feature heatmap, as illustrated in Figure \ref{fig:eps_example132123}. The dark blue points on the heatmap clearly indicate a strong inverse correlation between "voluntary surrender" and "confession". Additionally, a similar negative association is observed between "plea of guilt and acceptance of punishment" and "voluntary plea of guilt in court".

\begin{figure}[h]
    \centering
    \includegraphics[width=1\textwidth]{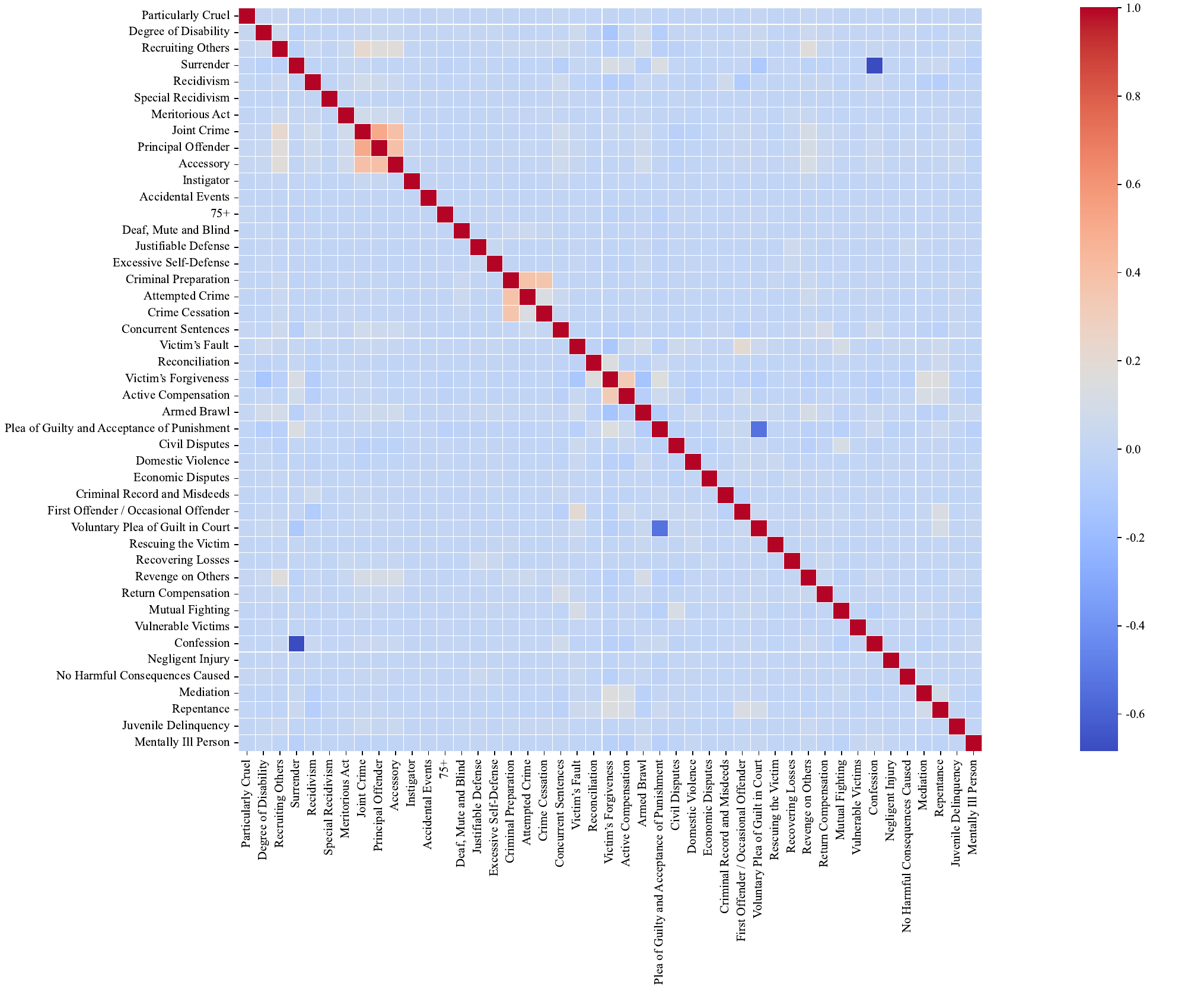}  
    \caption{Feature Correlation Heatmap}
    \label{fig:eps_example132123}
\end{figure}

To uncover the judicial interpretation behind this phenomenon, we examined the definitions of these four features, as outlined in Table \ref{feature-table12112}. Through consultations with legal practitioners, we learned that the features ``voluntary surrender" and ``confession", as well as "plea of guilt and acceptance of punishment" and ``voluntary plea of guilt in court", should not be redundantly evaluated. Specifically, when "voluntary surrender" is present, the feature ``voluntary  plea of guilt in court" should be set to 0. Similarly, ``plea of guilt and acceptance of punishment" does not overlap with ``voluntary plea of guilt in court". Only when ``voluntary surrender", ``confession", and ``plea of guilt and acceptance of punishment" are all set to 0 can "voluntary plea of guilt in court" be assigned a value of 1.

\newpage
\begin{table}[h]
\caption{The Judicial Interpretations for the Four Commonly Confused Features}
\label{feature-table12112}
\vskip 0.15in
\begin{center}
\begin{small}
\scalebox{0.94}{\begin{tabular}{lp{8cm}}
\toprule
Feature Name & Feature Definition \\
\midrule
Voluntary Surrender & The act of voluntarily turning oneself in after committing a crime and truthfully confessing one's actions. \\
Confession & The act of a suspect truthfully admitting to their crime even though they do not meet the criteria for voluntary surrender. \\
Voluntary Plea of Guilt in Court & The defendant voluntarily admits to the crime they are accused of during the court trial process. \\
Plea of Guilt and Acceptance of Punishment & The defendant truthfully confesses to their crime, raises no objection to the facts of the crime as charged, and agrees to the prosecutor's sentencing recommendation and signs the corresponding agreement. \\
\bottomrule
\end{tabular}}
\end{small}
\end{center}
\vskip -0.1in
\end{table}

Building on these principles, we generated new features by combining the four factors in various ways to better align our model with the actual sentencing process. The specific combinations are presented in Table \ref{feature-combination-table}.

\begin{table}[h]
\caption{Combination of Four Features.}
\label{feature-combination-table}
\vskip 0.15in
\begin{center}
\begin{tabular}{cccc}
\toprule
\scriptsize Voluntary Surrender & \scriptsize Confession & \scriptsize Plea of Guilt and Acceptance of Punishment & \scriptsize Voluntary Plea of Guilty in Court\\
\midrule
1 & 0 & 0 & 0 \\
1 & 0 & 1 & 0 \\
0 & 1 & 0 & 0 \\
0 & 1 & 1 & 0 \\
0 & 0 & 1 & 0 \\
0 & 0 & 0 & 1 \\
\bottomrule
\end{tabular}
\end{center}
\vskip -0.1in
\end{table}

\clearpage

\section{Randomized Search Space and Stability Assessment } \label{hyperanbhbu}

\begin{table*}[htbp]
\centering
\caption{Model configurations and hyperparameter settings}
\label{tab:model_configs_new11111}

\newcommand{\modelgroup}[1]{%
  \parbox{0.9\textwidth}{%
    \centering
    #1
  }%
}

\modelgroup{
\begin{tabular}{p{0.15\textwidth} p{0.45\textwidth} p{0.25\textwidth}}
\toprule
\textbf{Model} & \textbf{Tuned Hyperparameters} & \textbf{Default Parameters} \\
\midrule
Linear Regression & None & fit\_intercept: True \\
Random Forest &
\begin{tabular}[t]{@{}l@{}}
n\_estimators: 50--300 (int) \\
max\_depth: 3--15 (int) \\
min\_split: 2--11 (int) \\
min\_leaf: 1--6 (int) \\
bootstrap: [True, False]
\end{tabular} &
\begin{tabular}[t]{@{}l@{}}
n\_jobs: -1 \\
random\_state: 42
\end{tabular} \\
\bottomrule
\end{tabular}
}

\vspace{1em} 

\modelgroup{
\begin{tabular}{p{0.15\textwidth} p{0.45\textwidth} p{0.25\textwidth}}
\toprule
\textbf{Model} & \textbf{Tuned Hyperparameters} & \textbf{Default Parameters} \\
\midrule
MLP &
\begin{tabular}[t]{@{}l@{}}
hidden\_sizes: \{(64,), (128,), (64,64), (128,64)\} \\
activation: \{relu, tanh\} \\
alpha: $10^{-4}$ to $10^{-1}$ (log scale) \\
lr\_init: $10^{-3}$ to $10^{-1}$ (log scale)
\end{tabular} &
\begin{tabular}[t]{@{}l@{}}
early\_stopping: False \\
max\_iter: 500 \\
random\_state: 42
\end{tabular} \\
LightGBM &
\begin{tabular}[t]{@{}l@{}}
num\_leaves: 20--200 (int) \\
max\_depth: 3--12 (int) \\
learning\_rate: 0.01--0.3 \\
n\_estimators: 50--300 (int) \\
subsample: 0.6--1.0 \\
colsample: 0.6--1.0
\end{tabular} &
\begin{tabular}[t]{@{}l@{}}
objective: regression \\
random\_state: 42
\end{tabular} \\
\bottomrule
\end{tabular}
}

\vspace{1em} 

\modelgroup{
\begin{tabular}{p{0.15\textwidth} p{0.45\textwidth} p{0.25\textwidth}}
\toprule
\textbf{Model} & \textbf{Tuned Hyperparameters} & \textbf{Default Parameters} \\
\midrule
XGBoost &
\begin{tabular}[t]{@{}l@{}}
n\_estimators: 50--300 (int) \\
max\_depth: 3--12 (int) \\
learning\_rate: 0.01--0.3 \\
subsample: 0.6--1.0 \\
colsample: 0.6--1.0 \\
gamma: 0--0.5
\end{tabular} &
\begin{tabular}[t]{@{}l@{}}
tree\_method: hist \\
objective: reg:squarederror \\
eval\_metric: mae \\
random\_state: 42
\end{tabular} \\
CatBoost &
\begin{tabular}[t]{@{}l@{}}
iterations: 50--300 (int) \\
depth: 3--10 (int) \\
learning\_rate: 0.01--0.3 \\
l2\_reg: 1--5 \\
border\_count: 32--255 (int)
\end{tabular} &
\begin{tabular}[t]{@{}l@{}}
verbose: 0 \\
loss\_function: RMSE \\
random\_state: 42
\end{tabular} \\
\bottomrule
\end{tabular}
}

\footnotesize
Notation:  
(int): Discrete integer values in the given range.  
(log scale): Sampled logarithmically from the interval.  
[...]: List of discrete options.  
All other parameters use default values unless otherwise specified.
\end{table*}

\begin{table}[htbp]
  \centering
  \caption{Evaluation Results over 5 Different Random Seeds (Mean ± Std)}
  \label{tab:errorbar}
  \resizebox{\linewidth}{!}{ 
  \begin{tabular}{lccccccc}
    \toprule
    Metric & LR & RF & MLP & LGBM & XGB & CB & SMS\\
    \midrule
    RA(Serious) & 83.25 ± 0.50 & 86.83 ± 1.94 & 85.67 ± 1.76 & 86.41 ± 1.95 & 87.05 ± 1.87 & 87.36 ± 1.78 & 91.34 ± 0.06\\
    RA(Minors) & 73.15 ± 0.60 & 75.01 ± 0.43 & 74.25 ± 0.94 & 76.19 ± 0.14 & 76.23 ± 0.45 & 76.16 ± 0.29 & 77.53 ± 0.02\\
    \bottomrule
  \end{tabular}}
\end{table}

Table \ref{tab:errorbar} presents the evaluation results evaluated across five random seeds for each model, using their optimal hyperparameter configurations. Each reported value represents the mean and standard deviation (±) of the corresponding evaluation metric, reflecting the stability of the model's performance under stochastic variations (e.g., initialization or training noise). Because our method is an online approach requiring sequential processing of temporal data, traditional K-fold cross-validation (which involves shuffling) was unsuitable. Instead, we computed error bars by fixing the data split and varying only the model's internal random seeds, thereby isolating the impact of stochasticity while preserving temporal structure.

\end{document}